\documentclass[12pt,journal,onecolumn]{IEEEtran}
\usepackage{amsmath,amsfonts,amssymb}
\usepackage{algorithmic}
\usepackage[linesnumbered,ruled]{algorithm2e}
\usepackage[sort&compress, numbers]{natbib}
\usepackage{array}
\usepackage{textcomp}
\usepackage{stfloats}
\usepackage{url}
\usepackage{verbatim}
\usepackage{graphicx}
\usepackage{amsthm}
\usepackage{optidef}
\usepackage{subfigure}
\usepackage{xcolor}
\allowdisplaybreaks

\newtheorem{thm}{Theorem}
\newtheorem{ass}{Assumption}
\newtheorem{defi}{Definition}
\newtheorem{lem}{Lemma}

\newcommand{\norm}[1]{\left\lVert#1\right\rVert}
\begin{document}

\title{Minimax Optimal $Q$ Learning with Nearest Neighbors}

\author{Puning Zhao, Lifeng Lai\thanks{Puning Zhao is with Zhejiang Lab. Email: pnzhao1@gmail.com. Lifeng Lai is with Department of Electrical and Computer Engineering, University of California, Davis, CA, 95616. Email: lflai@ucdavis.edu. This work of L. Lai was supported by the National Science Foundation under grant CCF-21-12504.}}



\maketitle
\begin{abstract}
	Analyzing the Markov decision process (MDP) with continuous state spaces is generally challenging.
	A recent interesting work \cite{shah2018q} solves MDP with bounded continuous state space by a nearest neighbor $Q$ learning approach, which has a sample complexity of $\tilde{O}(\frac{1}{\epsilon^{d+3}(1-\gamma)^{d+7}})$ for $\epsilon$-accurate $Q$ function estimation with discount factor $\gamma$. In this paper, we propose two new nearest neighbor $Q$ learning methods, one for the offline setting and the other for the online setting. We show that the sample complexities of these two methods are $\tilde{O}(\frac{1}{\epsilon^{d+2}(1-\gamma)^{d+2}})$ and $\tilde{O}(\frac{1}{\epsilon^{d+2}(1-\gamma)^{d+3}})$ for offline and online methods respectively, which significantly improve over existing results and have minimax optimal dependence over $\epsilon$. We achieve such improvement by utilizing the samples more efficiently. In particular, the method in \cite{shah2018q} clears up all samples after each iteration, thus these samples are somewhat wasted. On the other hand, our offline method does not remove any samples, and our online method only removes samples with time earlier than $\beta t$ at time $t$ with $\beta$ being a tunable parameter, thus our methods significantly reduce the loss of information. Apart from the sample complexity, our methods also have additional advantages of better computational complexity, as well as suitability to unbounded state spaces.
	
\end{abstract}

\section{Introduction}

In nonparametric statistics, optimal rates have been established for various statistical tasks \cite{tsybakov2009introduction,yang1999minimax,scott2006minimax,raskutti2011minimax}, with most of them focusing on identical and independently distributed (i.i.d) data, while problems with non-i.i.d samples are rarely explored. Among these problems, the Markov decision process (MDP) is an important one, which is a stochastic control process that models various practical sequential decision making problems \cite{white1993survey,feinberg2012handbook,alsheikh15,bauerle2011markov,lauri2022partially}. In MDP, at each time step, an agent selects an action from a set $\mathcal{A}$ and then moves to another state and receives a reward. Compared with nonparametric estimation for i.i.d data \cite{tsybakov2009introduction,yang1999minimax,scott2006minimax,raskutti2011minimax} and MDP with finite state spaces \cite{even2003learning,beck2012error,chen2020finite,li2024q}, the design of learning algorithms for MDP with continuous state spaces faces the following two challenges. Firstly, states, actions, and rewards are collected sequentially. In early steps, estimates of the value function are inevitably inaccurate due to limited information. Since later estimates depend on earlier results, estimation errors in the early stages will have a negative impact on later estimates. A proper handling of early steps is thus crucially needed. Secondly, with a continuous state space, states do not appear repeatedly, thus the value function cannot be updated step-by-step as in the discrete state space. It is therefore necessary to design new update rules to use the information from neighboring states.  

Recently, \cite{shah2018q} proposed an interesting nonparametric method, called nearest neighbor $Q$ learning (NNQL) for MDP with continuous state spaces. To overcome the challenge that states do not repeat, NNQL divides the state space into many small regions, so that the estimation of the $Q$ function is based on previous samples falling in the same region. To avoid the impact caused by inaccurate estimation at early stages, NNQL clears up all samples after each iteration. With such a design, NNQL provides an $\ell_\infty$ consistent estimation of the optimal $Q$ function. Despite such progress, there are still some remaining problems that require further investigation. Firstly, the sample complexity is still not optimal. For $\epsilon$-accurate $Q$ function estimation under $\ell_\infty$ metric with discount factor $\gamma$, NNQL achieves a sample complexity $\tilde{O}\left(\frac{1}{\epsilon^{d+3} (1-\gamma)^{d+7}}\right)$ for a $d$ dimensional state space, while estimation with i.i.d samples only require $\tilde{O}(1/\epsilon^{d+2})$ samples \cite{jiang2019non}, indicating a potential room for further improvement. Intuitively speaking, to avoid the estimation error caused by early steps, NNQL clears up all samples after each iteration. Removal of early steps inevitably results in unnecessary loss of information and eventually leads to a suboptimal sample complexity. Secondly, NNQL discretizes the state space into a finite number of small regions, thus it is only suitable for bounded state spaces. However, practical MDP problems usually involve unbounded state spaces \cite{mobin2014information,he2015deep}. Although a relatively large estimation error is inevitable at the tail of state distribution, we hope to achieve a small average estimation error over the whole support set. 

In this paper, we propose two new nonparametric methods for $Q$ learning with nearest neighbors, one for the offline setting and the other one for the online setting. The offline algorithm starts after all samples are already collected. On the contrary, the online method updates the $Q$ function simultaneously as each state, action and reward are sequentially collected. There are two major differences with NNQL \cite{shah2018q}. Firstly, instead of dividing the support into regions as done in \cite{shah2018q}, our methods estimate $Q$ by directly averaging over neighboring states. As a result, our methods can be used in unbounded state spaces as well. Secondly, to improve the sample complexity, instead of clearing up samples after each iteration, we carefully design our methods to reuse samples from early steps. The offline method does not remove any samples throughout the whole training process, while the online method only removes steps earlier than $\beta t$ for some constant $\beta$. As a result, our methods use samples more efficiently. 

To illustrate the advantages of our approach, we conduct a theoretical analysis to analyze the sample complexities of the proposed methods. To begin with, we analyze the case where the state space is bounded. We obtain a high probability bound of the uniform convergence of $Q$ function estimation. We then analyze the more challenging case with unbounded state spaces. For the case with unbounded state spaces, the estimation error is always large at the tail of state distribution, thus uniform convergence is impossible. Therefore, we show a bound of the averaged estimation error weighted by the final stationary distribution. The result shows that the sample complexity is $\tilde{O}\left(\frac{1}{\epsilon^{d+2}(1-\gamma)^{d+2}}\right)$ for the offline method, and $\tilde{O}\left(\frac{1}{\epsilon^{d+2}(1-\gamma)^{d+3}}\right)$ for the online method. These two bounds have the same dependence on $\epsilon$. For the dependence on $1/(1-\gamma)$, the online method is slightly worse than the offline one. The sample complexities of both offline and online methods significantly improve over \cite{shah2018q} in the dependence of both $\epsilon$ and $1/(1-\gamma)$. Moreover, the dependence on $\epsilon$ matches the nonparametric rate for i.i.d samples \cite{tsybakov2009introduction}, and is thus optimal. 

Our contributions are summarized as follows.
\begin{itemize}
	\item For the offline setting, we propose a nearest neighbor $Q$ learning method, which iteratively refines the estimate of the $Q$ function. Throughout the training process, no samples are removed. 
	
	\item For the online setting, we propose another nearest neighbor $Q$ learning method. At the $t$-th step, it removes steps earlier than $\beta t$, in which $\beta$ needs to be tuned carefully to achieve a good tradeoff between reusing the information of early samples, and controlling the impact of inaccurate estimation at early steps.
	
	\item For both offline and online methods, we provide a theoretical analysis over bounded support first. We provide a uniform bound on the estimation error $\epsilon$ that holds with high probability. It turns out that the sample complexities are $\tilde{O}\left(\frac{1}{\epsilon^{d+2}(1-\gamma)^{d+2}}\right)$ and $\tilde{O}\left(\frac{1}{\epsilon^{d+2}(1-\gamma)^{d+3}}\right)$ for offline and online methods, respectively, which improve over existing method \cite{shah2018q} and have minimax optimal dependence on $\epsilon$. 
	
	\item The theoretical analysis is then generalized to unbounded support. While uniform convergence is impossible, we show that the average estimation error converges as fast as the case with bounded state support. This result indicates that compared with \cite{shah2018q} and other methods based on state space discretization \cite{sinclair2019adaptive,sinclair2023adaptive}, our methods are more suitable to unbounded state spaces.
\end{itemize}

In general, our analysis indicates that the new proposed methods have advantages in both sample complexity and the suitability to unbounded state spaces.

\section{Related Work}

\textbf{$Q$ learning for discrete state spaces.} $Q$ learning is a popular model-free reinforcement learning method to solve MDP with discrete state spaces \cite{dayan1992q}. Here we discuss the related work on $Q$ function estimation first. With this goal, it suffices to use a random exploration strategy. \cite{gheshlaghi2013minimax} shows that the minimax lower bound of sample complexity of $Q$ function estimation is $\Omega\left(\frac{|\mathcal{S}|}{\epsilon^2(1-\gamma)^3}\right)$, in which $|\mathcal{S}|$ is the size of state space. However, it is quite challenging to achieve this minimax lower bound. \cite{even2003learning} provides the first analysis on $Q$ learning, which shows that with a linear learning rate, the dependence on $1/(1-\gamma)$ is exponential. With a polynomial learning rate, the dependence on $\epsilon$ is suboptimal. The bound is then improved to $\tilde{O}\left(\frac{|\mathcal{S}|}{\epsilon^2(1-\gamma)^5}\right)$ in subsequent works \cite{beck2012error,wainwright2019stochastic,chen2020finite}. \cite{li2024q} further improves the bound to $\tilde{O}\left(\frac{|\mathcal{S}|}{\epsilon^2(1-\gamma)^4}\right)$, and show that this rate is tight. There are also some works that focus on improving exploration strategies to achieve optimal regrets, such as \cite{jin2018q,dong2019q,lakshmanan2015improved,bai2019provably,zhang2020almost,li2021breaking,he2021nearly}.

\textbf{$Q$ learning for continuous state spaces with parametric method.} This type of methods make some parametric function approximation, such as linear approximation \cite{melo2008analysis,chen2019performance,carvalho2020new,jin2020provably,wang2020reward,xiong2022nearly,he2023nearly} and neural network \cite{van2016deep,mnih2013playing,mnih2015human,fan2020theoretical,he2023nearly,zhang2023convergence}. While these methods have enjoyed great success in many practical problems \cite{duan2016benchmarking,mnih2015human,silver2016mastering}, the theoretical guarantees have not been well established. In particular, the $Q$ function may not lie within the parametric family determined by the model architecture. Therefore, these methods can not be used to approximate arbitrary $Q$ functions. As a result, the estimation error may not converge to zero even with the number of steps going to infinity, i.e. $T\rightarrow \infty$.

\textbf{Nonparametric minimax rates for i.i.d data.} Nonparametric statistical rates have been widely analyzed in various problems \cite{tsybakov2009introduction}. For nonparametric regression, the sample complexity of achieving $\epsilon$ error under $\ell_\infty$ metric is $\Omega(1/\epsilon^{d+2})$ \cite{stone1982optimal,raskutti2009lower,raskutti2011minimax,jiang2019non,zhao2021minimax}. Common nonparametric methods such as Nadaraya-Watson estimator \cite{nadaraya1964estimating} or $k$ nearest neighbor method \cite{biau2015lectures} can both achieve this rate. These analyses can not be directly used for solving MDP since samples are now sequentially dependent.

To the best of our knowledge, our work is the first attempt to achieve optimal sample complexity of estimating $Q$ function with respect to estimation error $\epsilon$ for MDP with continuous state spaces. Moreover, our work is also the first attempt to bound the sample complexity for unbounded continuous state spaces.

\section{Preliminaries}\label{sec:preliminary}
Consider an MDP $(\mathcal{S}, \mathcal{A}, p, r, \gamma)$, from which a sequence $(S_0,A_0, R_0), (S_1,A_1, R_1), (S_2, A_2, R_2), \ldots$ is generated. Here $\mathcal{S}$ is the state space, and $\mathcal{A}$ is the action space. In this paper, we assume that the cardinality of the state space $\mathcal{S}\subset \mathbb{R}^d$ is infinitely large, while $|\mathcal{A}|$ is finite. $p:\mathcal{S}\times \mathcal{A}\rightarrow \mathbb{R}^+$ is the transition kernel, such that $p(\cdot|s,a)$ is the probability density function (pdf) of $S_{t+1}$ conditional on $S_t=s$ and $A_t=a$. $r$ is the expected reward function. In this paper, we assume that the reward $R_t$ after taking action $A_t$ at state $S_t$ is
\begin{eqnarray}
	R_t = r(S_t, A_t) + W_t,
	\label{eq:R}
\end{eqnarray}
in which $W_t$ is the noise with zero expectation conditional on all the previous steps as well as the current state and action:
\begin{eqnarray}
	\mathbb{E}[W_t|S_1, A_1, R_1,\ldots, S_{t-1}, A_{t-1}, R_{t-1}, S_t, A_t] = 0.
\end{eqnarray}
Finally, $\gamma\in (0,1)$ is the discount factor. We are interested in the overall reward
\begin{eqnarray}
	G = \sum_{t=0}^\infty \gamma^t R_t.
\end{eqnarray}

A policy $\pi(\cdot|s)$ is the conditional probability mass function (pmf) of action $A_t$ given the state $S_t=s$. The $Q$ function is defined as
\begin{eqnarray}
	Q_\pi(s,a) = \mathbb{E}\left[\left. \sum_{t=0}^\infty \gamma^t R_t\right|S_0 = s, A_0 = a\right],
\end{eqnarray}
and denote $Q^*$ as the $Q$ function under the optimal policy, i.e.
\begin{eqnarray}
	Q^*(s,a) = \underset{\pi}{\sup} Q_\pi(s,a).
	\label{eq:Qstar}
\end{eqnarray}

Following existing research \cite{even2003learning,beck2012error,chen2020finite,li2024q}, our goal is to estimate the function $Q^*$ for all $s\in \mathcal{S}$ and $a\in \mathcal{A}$. In reinforcement learning, the ultimate goal is to identify the best policy, which has some difference with estimating $Q^*$. Nevertheless, the analysis of estimating $Q^*$ is still the focus of many existing research since the analysis reveals the complexity of learning MDP.

We now list basic assumptions used in our theoretical analysis for both offline and online methods. Throughout these assumptions, $\norm{\cdot}$ can be an arbitrary norm.

\begin{ass}\label{ass:main}
	Assume that there are some constants $R$, $L_r$, $\sigma$, $C_p$ and $\pi_0$, such that
	
	(a) The reward function $r(s,a)$ is bounded within $[0, R]$, and is $L_r$-Lipschitz with respect to $s$, i.e. for any $s, s', a$,
	\begin{eqnarray}
		|r(s, a)-r(s', a)|\leq L_r\norm{s-s'};
	\end{eqnarray}
	
	(b) The noise $W_t$ is subgaussian with parameter $\sigma^2$ conditional on previous trajectory, i.e.
	\begin{eqnarray}
		\mathbb{E}[e^{\lambda W_i}|S_1,A_1,R_1,\ldots, S_{t-1}, A_{t-1}, R_{t-1}, S_t, A_t]
		\leq \exp\left(\frac{1}{2}\lambda^2\sigma^2\right);
	\end{eqnarray}
	
	(c) The transition pdf satisfies $|p(y|s,a)-p(y|s',a)|\leq L_p(y)\|s-s'\|$ for some function $L_p$ and all $y,s,s'$, in which $L_p$ satisfies
	\begin{eqnarray}
		\int_\mathcal{S} L_p(y) dy \leq C_p;
	\end{eqnarray}
	
	(d) The behavior policy $\pi$ satisfies $\pi(a|s)\geq \pi_0$ for any $a\in \mathcal{A}$ and $s\in \mathcal{S}$;

\end{ass}

We now comment on these assumptions and compare them with assumptions made in \cite{shah2018q}. Assumption (a) requires that the reward function is bounded and Lipschitz continuous, which has also been made in \cite{shah2018q}. It is possible to relax it to $\gamma$-H{\"o}lder continuity with $\gamma\leq 1$. Assumption (b) is slightly weaker than \cite{shah2018q}, which assumes that $R_t$ is also bounded in $[0, R]$. Assumption (c) is exactly the same as Assumption (A4) in \cite{shah2018q}, which requires that the transition kernel is Lipschitz with respect to the current state. The Lipschitz assumption is also commonly used in other works about MDP with continuous state space \cite{dufour2012approximation}. Assumption (d) requires that the probabilities of all actions are bounded away from zero. This assumption ensures sufficient exploration. Since our current goal is to estimate $Q^*$, enough exploration is necessary so that the sequence can visit all state and action pairs. \cite{shah2018q} uses $\epsilon$-greedy policy, which is a special case of the policies satisfying Assumption (d).  

In this paper, we discuss two different cases: the case with bounded state spaces and the case with unbounded state spaces. For the former case, we list technical conditions in Assumption \ref{ass:bounded}.
\begin{ass}\label{ass:bounded}
	(For bounded state space) There are some constants $c$, $\alpha$, $C_S$, $D$ such that
	
	(e) For any $s, y\in \mathcal{S}$ and $a\in \mathcal{A}$, $p_\pi^m(y|s,a)\geq c$, in which $p_\pi^m$ is the $m$ step transition kernel, i.e. the conditional pdf of $S_{t+m}$ given $S_t=s$ and $A_t=a$ under policy $\pi$;
	
	(f) For $r\leq D$, $V(B(s, r)\cap \mathcal{S})\geq \alpha v_dr^d$, in which $B(s,r)$ means a ball centering at $s$ with radius $r$, $V$ denotes the volume (i.e. Lebesgue measure), $v_d$ is the volume of $d$ dimensional unit ball;
	
	(g) The covering number of $\mathcal{S}$ using balls with radius $r$ is bounded by
	\begin{eqnarray}
		n_c\leq \frac{C_S}{r^d} + 1.
	\end{eqnarray}
\end{ass}

Assumption (e) is the same as the assumption made in Corollary 1 in \cite{shah2018q}, which ensures the ergodicity, such that all states will be visited without waiting for a long time. Ergodicity is necessary since the estimated $Q$ function converges to the ground truth only if there are a sufficient number of samples around each state. Assumption (f) is our new assumption, which prevents the corner of the support from being too sharp. This assumption is implicitly made in \cite{shah2018q}, which assumes $\mathcal{S}=[0,1]^d$, and (f) is satisfied with $D=1$ and $\alpha=1/2^d$. Our assumption (f) relaxes it to a much broader collection. The same assumption is also used in nonparametric estimation for i.i.d samples \cite{zhao2020minimax,zhao2022analysis}. Assumption (g) assumes that $\mathcal{S}$ is compact, which has also been made in \cite{shah2018q}.

For the case with unbounded state spaces, define
\begin{eqnarray}
	g(s')=\underset{s}{\inf} p_\pi^m (s'|s),
	\label{eq:g}
\end{eqnarray}
in which $p_\pi^m$ is the $m$-step transition kernel with policy $\pi$. We then have the following assumption.

\begin{ass}\label{ass:tail}
	(For unbounded state space) Assume that there are some constants $C_g$, $D$, $\alpha$, $C_0$, such that
	
	(e') For all $s$, $a$,
	\begin{eqnarray}
		\int p(s'|s,a)g^{-\frac{1}{d}}(s')ds'\leq C_g,
		\label{eq:tail1}
	\end{eqnarray}
	and
	\begin{eqnarray}
		\int_{g(s)<t}p(s'|s,a)(\norm{s'}+1)ds\leq C_g t^{\frac{1}{d}};
		\label{eq:tail2}
	\end{eqnarray}
	
	(f') For any $r\leq D$, $s\in \mathcal{S}$,
	\begin{eqnarray}
		\int_{B(s,r)}g(u)du\geq \alpha v_d r^d g(y);
	\end{eqnarray}
	
	(g') $\mathbb{E}[\norm{S'}|s, a]\leq C_0$, in which $S'\sim p(\cdot|s, a)$.
\end{ass}

Assumption (e') requires that the tail of distribution can not be too strong. Estimating $Q$ at the tail of state distribution is harder than estimating $Q$ function at the center. Therefore, some restrictions on the tail behavior are needed. \eqref{eq:tail1} requires that $g(s')$ is not too small on average, and \eqref{eq:tail2} requires that if the current state is at the tail of state distribution (i.e. $g(s)<t$), then the next state will still fall at the center region with high probability. Assumption (f') is similar to Assumption \ref{ass:bounded}(f), which restricts the non-uniformity of the function $g$. Assumption (g') prevents the states from being too far away from each other.

\section{Offline Method}\label{sec:offline}
In this section, we present the proposed $Q$ learning method using nearest neighbors for the offline setting~\cite{levine2020offline,kidambi2020morel,kumar2020conservative,prudencio2023survey,lu2022challenges}. Consider a sequence $S_1, A_1, R_1, \ldots, S_T, A_T, R_T, S_{T+1}$ generated from an MDP $(\mathcal{S}, \mathcal{A}, p, r, \gamma)$ according to a policy $\pi$. Since the method is offline, in the remainder of Section~\ref{sec:offline}, we assume that the entire trajectory has been fully received before executing the algorithm.

To begin with, recall the Bellman equation:
\begin{eqnarray}
	Q^*(s,a) = r(s,a) + \gamma\mathbb{E}\left[\underset{a'}{\max}Q^*(S',a')|s,a\right],
	\label{eq:bellman}
\end{eqnarray}
in which $S'$ is a random state following $p(\cdot|s,a)$, with $p$ being the transition kernel.

As has been mentioned in Section \ref{sec:preliminary}, our goal is to estimate $Q^*$. 
As $r(s,a)$ and $p(\cdot|s,a)$ are both unknown, we use the information from the trajectory to obtain a rough estimate. Define
\begin{eqnarray}
	Q_i&:&\{1,\ldots, T\} \rightarrow \mathbb{R},\label{eq:Q}\\
	q_i&:&(\mathcal{S}, \mathcal{A})\rightarrow \mathbb{R},
	\label{eq:q}
\end{eqnarray}
for $i=1,\ldots, N$, which will be calculated during the learning process.

Here, $q_i$ is the estimated $Q^*$ over all $s\in \mathcal{S}$ and $a\in \mathcal{A}$. Furthermore, $Q_i$ can be viewed as another estimate of $Q^*$, such that $Q_i(t)$ approximates $Q^*(S_t,A_t)$. Initially, $Q_0(t)=0$ for all $t$, and $q_0(s,a)=0$ for all $s\in \mathcal{S}$ and $a\in \mathcal{A}$. The update rule at the $i$-th iteration is designed as follows. For all $t=1,\ldots, T-1$ and all $a\in \mathcal{A}$,
\begin{eqnarray}
	q_i(S_{t+1},a)&=& \frac{1}{k}\sum_{j\in \mathcal{N}(S_{t+1},a)}Q_{i-1}(j),
	\label{eq:step2}\\
	Q_i(t)&=& R_t+\gamma \underset{a}{\max}\;q_{i}(S_{t+1}, a),
	\label{eq:step1}	
\end{eqnarray}
in which $\mathcal{N}(s,a)$ is the set of indices of $k$ nearest neighbors of $s$ among all states in the dataset with action $a$, i.e. $\{S_j|A_j= a\}$. $Q_i$ and $q_i$ refer to the functions $Q$ and $q$ at the $i$-th step, respectively. \eqref{eq:step1} and \eqref{eq:step2} are repeated for $N$ iterations, i.e. $i=0,\ldots, N-1$, in order to let $Q$ and $q$ converge. After $N$ iterations, we then calculate the function $q$ for all queried pairs of states and actions, i.e.
\begin{eqnarray}
	q_N(s,a)=\frac{1}{k}\sum_{j\in \mathcal{N}(s,a)} Q_N(j).
	\label{eq:qn}
\end{eqnarray}

Then $q$ can be used as the final estimate of $Q^*$. The pseudo-code of our method is shown in Algorithm \ref{alg:Q}.
\begin{algorithm}[h]
	\caption{Nearest Neighbor $Q$ Learning: Offline Method}\label{alg:Q} 
	\begin{algorithmic}
		\STATE Input: MDP dynamics $(\mathcal{S}, \mathcal{A}, p, r, \gamma)$ with unknown $p$ and $r$, policy $\pi$, and parameter $k$, set of queried points $\mathcal{D}_{query}$
		\STATE Generate a sequence $S_1, A_1, R_1,\ldots, S_T, A_T, R_T, S_{T+1}$ according to policy $\pi$
		\STATE Initialize $Q_0(t)=0$ for all $t=1,\ldots, T$, $q_0(s,a)=0$ for all $s\in \mathcal{S}$ and $a\in \mathcal{A}$
		\FOR{$i=0,\ldots, N-1$}
		\FOR{$t=1,\ldots, T$}
		\FOR{$a\in \mathcal{A}$}
		\STATE Calculate $q_i(S_t, a)$ according to \eqref{eq:step2}
		\ENDFOR
		\STATE Calculate $Q_i(t)$ according to \eqref{eq:step1}
		\ENDFOR
		\ENDFOR
		\STATE Calculate $q_N(s,a)$ according to \eqref{eq:qn} for all queried $(s,a)\in \mathcal{D}_{query}$
		\STATE Output: $q_N(s,a)$ for all $(s,a)\in \mathcal{D}_{query}$
	\end{algorithmic}
\end{algorithm}

Practically, we can construct $|\mathcal{A}|$ kd-trees for nearest neighbor search \cite{abbasifard2014survey}, with each tree corresponding to one action. When a new state action pair $(S_t, A_t)$ is observed, we can push it into the tree corresponding to $A_t$. With $N$ iterations, the overall time complexity should be $O(NdT\ln T)$.

Now we provide a theoretical analysis of the proposed nearest neighbor $Q$ learning method in Algorithm \ref{alg:Q}. Recall the Bellman equation \eqref{eq:bellman}. As long as $\gamma\in (0,1)$, given $r(s,a)$ and $p(\cdot|s,a)$, the solution of \eqref{eq:bellman} named $Q^*$ is unique. We claim that with sufficiently large data size, after an infinite number of iterations, $q_N$ obtained in \eqref{eq:qn} is a good approximator of $Q^*$.

Define
\begin{eqnarray}
	Q=\underset{N\rightarrow \infty}{\lim}Q_N, q=\underset{N\rightarrow \infty}{\lim} q_N,
\end{eqnarray}
then
\begin{eqnarray}
	Q(t) &=& R_t +\gamma \underset{a}{\max} q(S_{t+1}, a),\label{eq:limit1}\\
	q(s, a) &=& \frac{1}{k}\sum_{j\in \mathcal{N}(s,a)}Q(j).\label{eq:limit2}
\end{eqnarray}
From \eqref{eq:limit1} and \eqref{eq:limit2},
\begin{eqnarray}
	q(s,a) =\frac{1}{k}\sum_{j\in \mathcal{N}(s,a)}\left[R_j+\gamma \underset{a'}{\max}q(S_{j+1}, a')\right].
	\label{eq:qstable}
\end{eqnarray}

We now compare \eqref{eq:qstable} with the Bellman equation \eqref{eq:bellman}, which will provide high-level ideas and conditions on the convergence of the proposed method:
\begin{itemize}
	\item The first term in \eqref{eq:bellman}, namely $r(s,a)$, is replaced by $\sum_{j\in \mathcal{N}(s,a)} R_j/k$ in \eqref{eq:qstable}. From \eqref{eq:R}, the difference between them is 
	\begin{eqnarray}
		\frac{1}{k}\sum_{j\in \mathcal{N}(s,a)}R_j - r(s,a) 
		&=& \frac{1}{k}\sum_{j\in \mathcal{N}(s,a)}(R_j-r(S_j, A_j))+ \frac{1}{k}\sum_{j\in \mathcal{N}(s,a)} (r(S_j, A_j)-r(s,a))\nonumber\\
		&=&\frac{1}{k}\sum_{j\in \mathcal{N}(s,a)} W_j +\frac{1}{k} \sum_{j\in \mathcal{N}(s,a)} (r(S_j, A_j)-r(s,a)).
		\label{eq:noise}
	\end{eqnarray}
	The first term in \eqref{eq:noise} converges to zero if we let $k$ grow with the total time step $T$. The second term in \eqref{eq:noise} converges to zero if $k$ grows slower than the total time step $T$ since the $j$-th nearest neighbor of $(s,a)$ will be closer to $(s,a)$ as $T$ increases. Therefore, if we ensure that $k$ grows with $T$ but $k/T$ goes to zero, then \eqref{eq:noise} converges to zero. 
	\item The second terms of \eqref{eq:qstable} and \eqref{eq:bellman} are also different. However, with the analysis similar to the first term, we can show that the difference converges to zero if $k$ increases with $T$ and $k/T$ goes to zero. Therefore, as long as the growth rate of $k$ with respect to $T$ is appropriate, $q$ will be closer to $Q^*$ as $T$ increases.
\end{itemize}

Building on these insights, we provide a formal analysis, and the results are shown in the following theorems. Theorem \ref{thm:offline} and \ref{thm:tail} show the convergence results for bounded and unbounded state spaces, respectively.
\begin{thm}\label{thm:offline}
	Under Assumptions \ref{ass:main} and \ref{ass:bounded}, let
	\begin{eqnarray}
		k\sim T^{2/(d+2)},
		\label{eq:kopt}
	\end{eqnarray}
	then there exists a constant $C_{off}$, such that the supremum error of Algorithm \ref{alg:Q} is bounded by
	\begin{eqnarray}
		\text{P}\left(\norm{q-Q^*}_\infty > C_{off} \frac{1}{1-\gamma} T^{-\frac{1}{d+2}}\ln T\right) = o(1),
		\label{eq:result}
	\end{eqnarray}
	in which
	$q = \underset{N\rightarrow \infty}{\lim} q_N$.
\end{thm}
\begin{proof}
	Please see Appendix \ref{sec:offlinepf} for the detailed proof.
\end{proof}

Theorem \ref{thm:offline} establishes the uniform convergence rate of $Q$ function estimation. The uniform convergence rate of nonparametric regression with $T$ i.i.d samples under Lipschitz continuity assumption is $O(T^{-\frac{1}{d+2}}\ln T)$ \cite{jiang2019non}. From \eqref{eq:result}, it can be observed that for $Q$ function estimation, the error only grows up to a $1/(1-\gamma)$ factor, while the dependence on the sample size remains the same. From \eqref{eq:result}, the sample complexity of estimation is
\begin{eqnarray}
	T=\tilde{O}\left(\frac{1}{\epsilon^{d+2}(1-\gamma)^{d+2}}\right).
	\label{eq:complexity}
\end{eqnarray}

We then move on to the analysis of Algorithm \ref{alg:Q} for unbounded support. It is impossible to achieve uniform convergence of $Q$ function estimation, since for an arbitrarily large number of steps $T$, the estimation of $Q$ is always not accurate at the tail of the distribution of states. Therefore, for the case with unbounded support, we evaluate the quality of estimation using average absolute estimation error weighted by the stationary state distribution. To be more precise, we show the following theorem.

\begin{thm}\label{thm:tail}
	Under Assumptions \ref{ass:main} and \ref{ass:tail}, let $k\sim T^{2/(d+2)}$, then there exists a constant $C_{off}'$, such that
	\begin{eqnarray}
		\int \mathbb{E}\left[\max_a|q(s, a)-Q^*(s,a)|\right] f_\pi(s)ds\leq C_{off}' \frac{1}{1-\gamma}T^{-\frac{1}{d+2}}\ln T,
	\end{eqnarray}
	in which $f_\pi$ is the pdf of the stationary distribution of states with policy $\pi$.
\end{thm}
The proof of Theorem \ref{thm:tail} is shown in Appendix \ref{sec:tail}. Let the average error be $\epsilon=\int \mathbb{E}[|q(s, a)-Q^*(s,a)|] f_\pi(s)ds$. Then the sample complexity can still be bounded by \eqref{eq:complexity}. The result indicates that under an appropriate tail assumption (i.e. Assumption \ref{ass:tail}(e')), the convergence rate of average estimation error is the same as the case with bounded state supports. An intuitive explanation is that while the estimation error is relatively large at the tail, since states fall in the tail with low probability, the average estimation error does not increase significantly. Assumption \ref{ass:tail}(e') may be relaxed, and then the sample complexity may be higher. In general, our theoretical analysis shows that compared with discretization based approaches \cite{shah2018q,sinclair2023adaptive}, our method is more suitable to unbounded state spaces.


\section{Online Method}\label{sec:online}
In this section, we extend our study to the online setting. In the offline case discussed in Section~\ref{sec:offline}, the algorithm is executed after the whole trajectory is collected. On the contrary, in online learning, we need to update the model immediately after receiving each sample. At each time step $t$, we can not observe any information after $t$, thus the estimation of $Q^*$ must rely on earlier steps. Moreover, in the offline setting, evaluation with a set of query points is after the whole training process is finished. However, in online learning, a query request at state $s$ can occur at an arbitrary time. Due to such differences, we modify the offline nearest neighbor $Q$ learning method in Section \ref{sec:offline} to make it suitable for online problems.

We still define two functions $Q:\{1,\ldots. T\}\rightarrow \mathbb{R}$ and $q_t: (\mathcal{S}, \mathcal{A})\rightarrow \mathbb{R}$, for $t=1,\ldots, T$. The definition of $Q$ is exactly the same as \eqref{eq:Q} for the offline method. However, $q_t$ is slightly different from \eqref{eq:q}. In the online method, consider that the estimation of $Q^*$ is updated whenever a new sample is received instead of using all samples together, we use subscript $t$ in $q_t$ to denote the estimated $Q^*$ at iteration $t$. 

In each iteration, the agent starts from state $S_t$, takes action $A_t$ according to policy $\pi$, and then receives reward $R_t$ and next state $S_{t+1}$. The estimated $Q$ function is updated using the following rules:
\begin{eqnarray}
	q_t(S_{t+1},a) &=& \frac{1}{k(t)}\sum_{j\in \mathcal{N}_t(S_{t+1},a)} Q(j),\label{eq:qonline}	\\
	Q(t)&=& R_t+\gamma \underset{a}{\max}q_t(S_{t+1}, a),\label{eq:Qonline}
\end{eqnarray}
in which $k(t)$ is a list of parameters for $t=1,\ldots, T$. To make the learning consistent, $k(t)$ needs to grow with $t$ at an appropriate growth rate. $\mathcal{N}_t(s,a)$ is the set of $k(t)$ nearest neighbors of $s$ among $\{S_j|\beta t\leq j < t, A_i=a \}$. $\beta\in (0,1)$ is a hyperparameter.

In the online setting, at time step $t$, we only use steps after $\beta t$ to estimate $q_t(S_{t+1}, a)$. An intuitive explanation is that the estimation errors at early steps can be large, thus $Q(j)$ is not a good approximation of $Q^*(S_j, A_j)$ for small $j$. $\beta$ needs to be large enough to avoid the negative impact of estimation caused by early steps. However, if $\beta$ is too close to $1$, then there may not be enough samples in $\{S_j|\beta t\leq j<t, A_i=a\}$, thus the nearest neighbor distances can be large, which may increase the bias of $q_t(S_{t+1}, a)$. Therefore, $\beta$ should be chosen carefully to strike a tradeoff between reusing early samples and avoiding the impact of inaccurate estimation at early steps.

Finally, when there is a query at some state $s$ and action $a$ at time $t$, the algorithm returns
\begin{eqnarray}
	q_t(s,a)=\frac{1}{k(t)}\sum_{j\in \mathcal{N}_t(s,a)}Q(j)
	\label{eq:output}
\end{eqnarray}
as the estimated $Q^*$ function. 

There are several differences between the online and offline methods. Firstly, in the offline method, the values of $Q(t)$ and $q_t$ are updated with $N$ iterations (eq.\eqref{eq:step2} and \eqref{eq:step1}), while in the online method, \eqref{eq:qonline} and \eqref{eq:Qonline} only run once. This ensures that the computation is efficient. Secondly, for the offline method, \eqref{eq:step2}, $q_t(S_{t+1}, a)$ is calculated by averaging among $\mathcal{N}(S_{t+1}, a)$, while \eqref{eq:qonline} changes it to $\mathcal{N}_t(S_{t+1}, a)$ for the online method. Compared with $\mathcal{N}(S_{t+1}, a)$, $\mathcal{N}_t(S_{t+1}, a)$ does not consider steps $j\geq t$ and $j<\beta t$. In online reinforcement learning, we can not observe the trajectory after the current time step, thus all indices $j$ larger than $t$ are not included in \eqref{eq:qonline}, thus steps with $j> t$ can not be used. As discussed earlier, we remove samples with $j<\beta t$ to control the negative impact caused by inaccurate estimation at early steps. Therefore, in \eqref{eq:qonline}, we only use $Q(j)$ with $\beta t\leq j< t$ to calculate the value of $q$ using nearest neighbors. 

The procedure for online $Q$ learning is shown in Algorithm \ref{alg:Qonline}.
\begin{algorithm}[h]
	\caption{Nearest Neighbor $Q$ Learning: Online Method}\label{alg:Qonline}
	\begin{algorithmic}
		\STATE Input: MDP dynamics $(\mathcal{S}, \mathcal{A}, p, r, \gamma)$, with unknown $p$ and $r$, policy $\pi$, and parameter $k(t)$, $\beta$, and initial state $S_1$
		\STATE Initialize $q(S_0,a) = 0$ for all $a\in \mathcal{A}$
		\FOR{$t=1,\ldots, T$}
		\STATE Take action $A_t$ according to $\pi(\cdot|S_t)$
		\STATE Receive $R_t$ and $S_{t+1}$	
		\FOR{$a\in \mathcal{A}$}
		\STATE Calculate $q_t(S_{t+1}, a)$ according to \eqref{eq:qonline}
		\ENDFOR 
		\STATE Calculate $Q(t)$ according to \eqref{eq:Qonline}
		\IF{Received a query request at $(s, a)$}
		\STATE Output $q_t(s,a)$ according to \eqref{eq:output}		
		\ENDIF		
		\ENDFOR
	\end{algorithmic}
\end{algorithm}
Unlike the offline method, the computation can not rely on kd-trees since data become dynamic, with new samples coming in each iteration, while old samples may be removed. Hence, we use some new methods, such as R-tree \cite{abbasifard2014survey}. It turns out that the time complexity is $O(d\ln t)$ for each time step, and the overall time complexity after $T$ steps is $O(Td\ln T)$. 

Now we provide a theoretical analysis of the online method. For the offline method, we have analyzed the performance after infinite iterations, such that $Q$ and $q$ satisfy the relation \eqref{eq:limit1} and \eqref{eq:limit2}. However, for the online method, $Q(t)$ and $q(S_{t+1}, a)$ are calculated only once. Therefore, we need to use different analysis techniques. The result is shown in Theorem \ref{thm:online}.

\begin{thm}\label{thm:online}
	Under Assumptions \ref{ass:main} and \ref{ass:tail}, if $k(t)=\lceil ((1-\beta) t)^{2/(d+2)}\rceil$, $\beta = \gamma^\frac{d+2}{d+3}$, then there exists a constant $C_{on}$, such that the supremum error of Algorithm \ref{alg:Qonline} is bounded by
	\begin{eqnarray}
		\text{P}\left(\|q_T-Q^*\|_\infty >C_{on}(1-\gamma)^{-\frac{d+3}{d+2}}T^{-\frac{1}{d+2}}\ln T\right) = o(1).
		\label{eq:result2}
	\end{eqnarray}
\end{thm}
\begin{proof}
	(Outline) Define the supremum error $\Delta_t = \norm{q_t-Q^*}_\infty$.  To bound $\Delta_t$, recall \eqref{eq:qonline}, in which $q_t$ is calculated by the average of $k$ nearest neighbor of $Q(j)$. $\Delta_t$ can then be obtained by bounding the estimation error of $Q(j)$. Moreover, from \eqref{eq:Qonline}, the error of $Q(j)$ also relies on the error of $q_j$, with $\beta t \leq j< t$. With these two conversions \eqref{eq:qonline} and \eqref{eq:Qonline}, $\Delta_t$ can be bounded using the error bounds of earlier steps $\Delta_j$, as well as a uniform bound on the noise after nearest neighbor averaging. This yields an inequality (\eqref{eq:transition} in Section~\ref{sec:onlinepf} in the appendix), which characterizes how the estimation error decays step by step. Using this inequality, we use mathematical induction to obtain a bound of $\Delta_t$.	Please see Appendix \ref{sec:onlinepf} for detailed proof.
\end{proof}


From \eqref{eq:result2}, the sample complexity is bounded by
\begin{eqnarray}
	T=\tilde{O}\left(\frac{1}{\epsilon^{d+2}(1-\gamma)^{d+3}}\right).
	\label{eq:complexityonline}
\end{eqnarray}


We then generalize the analysis to the case with unbounded state support. The result is shown in Theorem \ref{thm:onlinetail}. The optimal parameters $k(t)$ and $\beta$ remain the same as the case with bounded support. 

\begin{thm}\label{thm:onlinetail}
	For online $Q$ learning, for small $1-\gamma$, let $k(t)=\lceil ((1-\beta) t)^{2/(d+2)}\rceil$, $\beta=\gamma^{(d+2)/(d+3)}$. Then under Assumptions \ref{ass:main} and \ref{ass:tail}, there exists a constant $C_{on}'$, such that
	\begin{eqnarray}
		\int \mathbb{E}\left[\max_a|q_T(s, a)-Q^*(s, a)|\right] f_\pi(s)ds\lesssim C_{on}' (1-\gamma)^{-\frac{d+3}{d+2}} T^{-\frac{1}{d+2}}\ln T.
		\label{eq:onlinetail}
	\end{eqnarray}
\end{thm}

The proof of Theorem \ref{thm:onlinetail} is shown in Appendix \ref{sec:onlinetail}. Similar to the offline $Q$ learning, due to a relatively large estimation error at the tail of state distribution, uniform convergence is impossible. Therefore, we bound the average estimation error weighted by the pdf of stationary distribution $f_\pi(s)$. Let the average error $\epsilon$ be the left hand side of \eqref{eq:onlinetail}, then the corresponding sample complexity is still bounded by \eqref{eq:complexityonline}. Therefore, the online method is also suitable for unbounded state spaces.

Finally, we compare the sample complexity \eqref{eq:complexityonline} with the result of the offline method \eqref{eq:complexity}. The dependence over $\epsilon$ remains the same.  As discussed earlier, after removing $\beta t$ steps, there are still $(1-\beta) t$ samples for calculating $q_t$ at time $t$. If $\beta$ is regarded as a constant, then the convergence of supremum estimation error with respect to $T$ remains the same. Therefore, the dependence of sample complexity over $\epsilon$ is not changed compared with the offline method. However, the dependence of sample complexity on $1-\gamma$ is worse than the offline one by a factor $1/(1-\gamma)$. Intuitively, this is because the online method removes some early samples. To be more precise, the offline method uses all steps $j=1,\ldots, T$ to estimate $Q^*(s, a)$ for each $s, a$, while the online method only uses from $\beta t$ to $t$. With optimal $\beta$, the online method only uses a $1-\gamma$ fraction of all samples on average, thus the overall sample complexity is $1/(1-\gamma)$ times larger than that of the offline method.
\section{Discussion}\label{sec:dis}

\subsection{Comparison with \cite{shah2018q}}
There are several major differences between our method and NNQL \cite{shah2018q}. NNQL divides the state space into many small regions with fixed bandwidth parameter $h$, and the estimated $Q(S_{t+1}, a)$ is averaged over all samples that fall in the same region with $S_{t+1}, a$. After each region is occupied by at least one sample, the counts of samples in all regions are reset to zero, which means that all existing samples are removed. We compare our method with NNQL in the following aspects.

\begin{itemize}
	\item Sample complexity. According to Corollary 1 of \cite{shah2018q}, the sample complexity of achieving $\epsilon$-accurate estimation of $Q^*$ is
	\begin{eqnarray}
		T = \tilde{O}\left(\frac{1}{\epsilon^{d+3} (1-\gamma)^{d+7}}\right).
		\label{eq:previous}
	\end{eqnarray}
	From \eqref{eq:result} and \eqref{eq:result2}, both our offline and online methods improve over \eqref{eq:previous}. The intuitive reason is that our offline method does not remove any samples, while the online method only removes steps earlier than $\beta t$ at time $t$ to reduce the influence of inaccurate $Q$ function estimation at early stages. Therefore, we use samples more efficiently.
	
	\item Computational complexity. With the increase of dimensionality, the number of regions of NNQL grows exponentially, which leads to a large computation cost. Instead, we use a direct nearest neighbor approach, and the computational cost only grows linearly with $d$.
	
	\item Suitability to unbounded support. Since our method does not rely on state space discretization, our method can be generalized to unbounded state spaces. If the tail is not too heavy (which is stated precisely in Assumption \ref{ass:tail}(e')), then the convergence rate of average estimation error remains the same as the case with bounded support.
\end{itemize}

\subsection{Comparison with the minimax lower bound}

A simple way to obtain the minimax lower bound is to just let $p(s'|s,a)$ be the same for all $s,a$. Then the $Q$ learning problem is converted to nonparametric regression. According to \cite{stone1982optimal}, for any $\delta\in (0,1)$, there exists a function $f$ such that the $\ell_\infty$ estimation error is at least $\Omega\left((\ln T/T)^{1/(d+2)}\right)$. Therefore, for all estimator $\hat{Q}$ and for all $\delta\in(0,1)$, there exists an MDP problem such that
\begin{eqnarray}
	\text{P}\left(\norm{\hat{Q}-Q}_\infty \geq C\left(\frac{\ln T}{T}\right)^\frac{1}{2+d}\right)\geq \delta,
	\label{eq:mmx}
\end{eqnarray}
in which $C$ is a constant. From \eqref{eq:mmx}, the sample complexity of estimating $Q$ is at least $\Omega(1/\epsilon^{d+2})$. Therefore, compared with \eqref{eq:mmx}, both our offline and online methods are nearly minimax optimal in the dependence on $\epsilon$. It is not clear whether the sample complexity \eqref{eq:complexity} is also optimal in the dependence over $1/(1-\gamma)$, which is an interesting future work.

\section{Conclusion}\label{sec:conc}
In this paper, we have proposed two $Q$ learning methods for continuous state space based on $k$ nearest neighbor. One of them is offline, while the other is online. These methods can be used to estimate the optimal $Q$ function of MDPs. We have also conducted a theoretical analysis to bound the convergence rate of the estimated $Q$ function to the ground truth. The result shows that the sample complexity of both offline and online methods have optimal dependence of estimation error $\epsilon$. Compared with previous works, our new methods significantly improve the convergence rate, as we use training samples more efficiently. 

\bibliographystyle{ieeetr}
\bibliography{QLearning}

\begin{thebibliography}{10}

\bibitem{shah2018q}
D.~Shah and Q.~Xie, ``{Q}-learning with nearest neighbors,'' in {\em Advances
  in Neural Information Processing Systems}, pp.~3111--3121, 2018.

\bibitem{tsybakov2009introduction}
A.~B. Tsybakov, {\em Introduction to Nonparametric Estimation}.
\newblock 2009.

\bibitem{yang1999minimax}
Y.~Yang, ``Minimax nonparametric classification. i. rates of convergence,''
  {\em IEEE Transactions on Information Theory}, vol.~45, no.~7,
  pp.~2271--2284, 1999.

\bibitem{scott2006minimax}
C.~Scott and R.~D. Nowak, ``Minimax-optimal classification with dyadic decision
  trees,'' {\em IEEE Transactions on Information Theory}, vol.~52, no.~4,
  pp.~1335--1353, 2006.

\bibitem{raskutti2011minimax}
G.~Raskutti, M.~J. Wainwright, and B.~Yu, ``Minimax rates of estimation for
  high-dimensional linear regression over $\ell_q$-balls,'' {\em IEEE
  Transactions on Information Theory}, vol.~57, no.~10, pp.~6976--6994, 2011.

\bibitem{white1993survey}
D.~J. White, ``A survey of applications of markov decision processes,'' {\em
  Journal of the operational research society}, vol.~44, no.~11,
  pp.~1073--1096, 1993.

\bibitem{feinberg2012handbook}
E.~A. Feinberg and A.~Shwartz, {\em Handbook of Markov decision processes:
  methods and applications}, vol.~40.
\newblock Springer Science \& Business Media, 2012.

\bibitem{alsheikh15}
M.~Abu~Alsheikh, D.~T. Hoang, D.~Niyato, H.-P. Tan, and S.~Lin, ``Markov
  decision processes with applications in wireless sensor networks: A survey,''
  {\em IEEE Communications Surveys \& Tutorials}, vol.~17, no.~3,
  pp.~1239--1267, 2015.

\bibitem{bauerle2011markov}
N.~B{\"a}uerle and U.~Rieder, {\em Markov decision processes with applications
  to finance}.
\newblock Springer Science \& Business Media, 2011.

\bibitem{lauri2022partially}
M.~Lauri, D.~Hsu, and J.~Pajarinen, ``Partially observable markov decision
  processes in robotics: A survey,'' {\em IEEE Transactions on Robotics},
  vol.~39, no.~1, pp.~21--40, 2022.

\bibitem{even2003learning}
E.~Even-Dar and Y.~Mansour, ``Learning rates for {Q}-learning,'' {\em Journal
  of machine learning Research}, vol.~5, no.~Dec, pp.~1--25, 2003.

\bibitem{beck2012error}
C.~L. Beck and R.~Srikant, ``Error bounds for constant step-size q-learning,''
  {\em Systems \& control letters}, vol.~61, no.~12, pp.~1203--1208, 2012.

\bibitem{chen2020finite}
Z.~Chen, S.~T. Maguluri, S.~Shakkottai, and K.~Shanmugam, ``Finite-sample
  analysis of stochastic approximation using smooth convex envelopes,'' {\em
  arXiv preprint arXiv:2002.00874}, 2020.

\bibitem{li2024q}
G.~Li, C.~Cai, Y.~Chen, Y.~Wei, and Y.~Chi, ``Is q-learning minimax optimal? a
  tight sample complexity analysis,'' {\em Operations Research}, vol.~72,
  no.~1, pp.~222--236, 2024.

\bibitem{jiang2019non}
H.~Jiang, ``Non-asymptotic uniform rates of consistency for k-nn regression,''
  in {\em Proceedings of the AAAI Conference on Artificial Intelligence},
  vol.~33, pp.~3999--4006, 2019.

\bibitem{mobin2014information}
S.~A. Mobin, J.~A. Arnemann, and F.~Sommer, ``Information-based learning by
  agents in unbounded state spaces,'' {\em Advances in Neural Information
  Processing Systems}, vol.~27, 2014.

\bibitem{he2015deep}
J.~He, J.~Chen, X.~He, J.~Gao, L.~Li, L.~Deng, and M.~Ostendorf, ``Deep
  reinforcement learning with an unbounded action space,'' {\em arXiv preprint
  arXiv:1511.04636}, vol.~5, 2015.

\bibitem{sinclair2019adaptive}
S.~R. Sinclair, S.~Banerjee, and C.~L. Yu, ``Adaptive discretization for
  episodic reinforcement learning in metric spaces,'' {\em Proceedings of the
  ACM on Measurement and Analysis of Computing Systems}, vol.~3, no.~3,
  pp.~1--44, 2019.

\bibitem{sinclair2023adaptive}
S.~R. Sinclair, S.~Banerjee, and C.~L. Yu, ``Adaptive discretization in online
  reinforcement learning,'' {\em Operations Research}, vol.~71, no.~5,
  pp.~1636--1652, 2023.

\bibitem{dayan1992q}
P.~Dayan and C.~Watkins, ``Q-learning,'' {\em Machine learning}, vol.~8, no.~3,
  pp.~279--292, 1992.

\bibitem{gheshlaghi2013minimax}
M.~Gheshlaghi~Azar, R.~Munos, and H.~J. Kappen, ``Minimax pac bounds on the
  sample complexity of reinforcement learning with a generative model,'' {\em
  Machine learning}, vol.~91, pp.~325--349, 2013.

\bibitem{wainwright2019stochastic}
M.~J. Wainwright, ``Stochastic approximation with cone-contractive operators:
  Sharp $\ell_\infty$-bounds for $ q $-learning,'' {\em arXiv preprint
  arXiv:1905.06265}, 2019.

\bibitem{jin2018q}
C.~Jin, Z.~Allen-Zhu, S.~Bubeck, and M.~I. Jordan, ``Is {Q}-learning provably
  efficient?,'' in {\em Advances in Neural Information Processing Systems},
  pp.~4863--4873, 2018.

\bibitem{dong2019q}
K.~Dong, Y.~Wang, X.~Chen, and L.~Wang, ``{Q}-learning with ucb exploration is
  sample efficient for infinite-horizon mdp,'' {\em arXiv preprint
  arXiv:1901.09311}, 2019.

\bibitem{lakshmanan2015improved}
K.~Lakshmanan, R.~Ortner, and D.~Ryabko, ``Improved regret bounds for
  undiscounted continuous reinforcement learning,'' in {\em International
  conference on machine learning}, pp.~524--532, PMLR, 2015.

\bibitem{bai2019provably}
Y.~Bai, T.~Xie, N.~Jiang, and Y.-X. Wang, ``Provably efficient q-learning with
  low switching cost,'' {\em Advances in Neural Information Processing
  Systems}, vol.~32, 2019.

\bibitem{zhang2020almost}
Z.~Zhang, Y.~Zhou, and X.~Ji, ``Almost optimal model-free reinforcement
  learningvia reference-advantage decomposition,'' {\em Advances in Neural
  Information Processing Systems}, vol.~33, pp.~15198--15207, 2020.

\bibitem{li2021breaking}
G.~Li, L.~Shi, Y.~Chen, Y.~Gu, and Y.~Chi, ``Breaking the sample complexity
  barrier to regret-optimal model-free reinforcement learning,'' {\em Advances
  in Neural Information Processing Systems}, vol.~34, pp.~17762--17776, 2021.

\bibitem{he2021nearly}
J.~He, D.~Zhou, and Q.~Gu, ``Nearly minimax optimal reinforcement learning for
  discounted mdps,'' {\em Advances in Neural Information Processing Systems},
  vol.~34, pp.~22288--22300, 2021.

\bibitem{melo2008analysis}
F.~S. Melo, S.~P. Meyn, and M.~I. Ribeiro, ``An analysis of reinforcement
  learning with function approximation,'' in {\em Proceedings of the 25th
  International Conference on Machine learning}, pp.~664--671, 2008.

\bibitem{chen2019performance}
Z.~Chen, S.~Zhang, T.~T. Doan, S.~T. Maguluri, and J.-P. Clarke, ``Performance
  of {Q}-learning with linear function approximation: {S}tability and
  finite-time analysis,'' {\em arXiv preprint arXiv:1905.11425}, 2019.

\bibitem{carvalho2020new}
D.~Carvalho, F.~S. Melo, and P.~Santos, ``A new convergent variant of
  q-learning with linear function approximation,'' {\em Advances in Neural
  Information Processing Systems}, vol.~33, pp.~19412--19421, 2020.

\bibitem{jin2020provably}
C.~Jin, Z.~Yang, Z.~Wang, and M.~I. Jordan, ``Provably efficient reinforcement
  learning with linear function approximation,'' in {\em Conference on Learning
  Theory}, pp.~2137--2143, PMLR, 2020.

\bibitem{wang2020reward}
R.~Wang, S.~S. Du, L.~Yang, and R.~R. Salakhutdinov, ``On reward-free
  reinforcement learning with linear function approximation,'' {\em Advances in
  Neural Information Processing Systems}, vol.~33, pp.~17816--17826, 2020.

\bibitem{xiong2022nearly}
W.~Xiong, H.~Zhong, C.~Shi, C.~Shen, L.~Wang, and T.~Zhang, ``Nearly minimax
  optimal offline reinforcement learning with linear function approximation:
  Single-agent mdp and markov game,'' {\em arXiv preprint arXiv:2205.15512},
  2022.

\bibitem{he2023nearly}
J.~He, H.~Zhao, D.~Zhou, and Q.~Gu, ``Nearly minimax optimal reinforcement
  learning for linear markov decision processes,'' in {\em International
  Conference on Machine Learning}, pp.~12790--12822, 2023.

\bibitem{van2016deep}
H.~Van~Hasselt, A.~Guez, and D.~Silver, ``Deep reinforcement learning with
  double {Q}-learning,'' in {\em Proceedings of the AAAI conference on
  artificial intelligence}, vol.~30, 2016.

\bibitem{mnih2013playing}
V.~Mnih, K.~Kavukcuoglu, D.~Silver, A.~Graves, I.~Antonoglou, D.~Wierstra, and
  M.~Riedmiller, ``Playing atari with deep reinforcement learning,'' {\em arXiv
  preprint arXiv:1312.5602}, 2013.

\bibitem{mnih2015human}
V.~Mnih, K.~Kavukcuoglu, D.~Silver, A.~A. Rusu, J.~Veness, M.~G. Bellemare,
  A.~Graves, M.~Riedmiller, A.~K. Fidjeland, G.~Ostrovski, {\em et~al.},
  ``Human-level control through deep reinforcement learning,'' {\em Nature},
  vol.~518, no.~7540, pp.~529--533, 2015.

\bibitem{fan2020theoretical}
J.~Fan, Z.~Wang, Y.~Xie, and Z.~Yang, ``A theoretical analysis of deep
  {Q}-learning,'' in {\em Learning for Dynamics and Control}, pp.~486--489,
  PMLR, 2020.

\bibitem{zhang2023convergence}
S.~Zhang, H.~Li, M.~Wang, M.~Liu, P.-Y. Chen, S.~Lu, S.~Liu, K.~Murugesan, and
  S.~Chaudhury, ``On the convergence and sample complexity analysis of deep
  q-networks with $\epsilon$-greedy exploration,'' {\em Advances in Neural
  Information Processing Systems}, vol.~36, 2023.

\bibitem{duan2016benchmarking}
Y.~Duan, X.~Chen, R.~Houthooft, J.~Schulman, and P.~Abbeel, ``Benchmarking deep
  reinforcement learning for continuous control,'' in {\em International
  Conference on Machine Learning}, pp.~1329--1338, 2016.

\bibitem{silver2016mastering}
D.~Silver, A.~Huang, C.~J. Maddison, A.~Guez, L.~Sifre, G.~Van Den~Driessche,
  J.~Schrittwieser, I.~Antonoglou, V.~Panneershelvam, M.~Lanctot, {\em et~al.},
  ``Mastering the game of {G}o with deep neural networks and tree search,''
  {\em Nature}, vol.~529, no.~7587, pp.~484--489, 2016.

\bibitem{stone1982optimal}
C.~J. Stone, ``Optimal global rates of convergence for nonparametric
  regression,'' {\em The Annals of Statistics}, pp.~1040--1053, 1982.

\bibitem{raskutti2009lower}
G.~Raskutti, B.~Yu, and M.~J. Wainwright, ``Lower bounds on minimax rates for
  nonparametric regression with additive sparsity and smoothness,'' {\em
  Advances in Neural Information Processing Systems}, vol.~22, 2009.

\bibitem{zhao2021minimax}
P.~Zhao and L.~Lai, ``Minimax rate optimal adaptive nearest neighbor
  classification and regression,'' {\em IEEE Transactions on Information
  Theory}, vol.~67, no.~5, pp.~3155--3182, 2021.

\bibitem{nadaraya1964estimating}
E.~A. Nadaraya, ``On estimating regression,'' {\em Theory of Probability \& Its
  Applications}, vol.~9, no.~1, pp.~141--142, 1964.

\bibitem{biau2015lectures}
G.~Biau and L.~Devroye, {\em Lectures on the nearest neighbor method},
  vol.~246.
\newblock Springer, 2015.

\bibitem{dufour2012approximation}
F.~Dufour and T.~Prieto-Rumeau, ``Approximation of markov decision processes
  with general state space,'' {\em Journal of Mathematical Analysis and
  applications}, vol.~388, no.~2, pp.~1254--1267, 2012.

\bibitem{zhao2020minimax}
P.~Zhao and L.~Lai, ``Minimax optimal estimation of kl divergence for
  continuous distributions,'' {\em IEEE Transactions on Information Theory},
  vol.~66, no.~12, pp.~7787--7811, 2020.

\bibitem{zhao2022analysis}
P.~Zhao and L.~Lai, ``Analysis of knn density estimation,'' {\em IEEE
  Transactions on Information Theory}, vol.~68, no.~12, pp.~7971--7995, 2022.

\bibitem{levine2020offline}
S.~Levine, A.~Kumar, G.~Tucker, and J.~Fu, ``Offline reinforcement learning:
  Tutorial, review, and perspectives on open problems,'' {\em arXiv preprint
  arXiv:2005.01643}, 2020.

\bibitem{kidambi2020morel}
R.~Kidambi, A.~Rajeswaran, P.~Netrapalli, and T.~Joachims, ``Morel: Model-based
  offline reinforcement learning,'' {\em Advances in Neural Information
  Processing Systems}, vol.~33, pp.~21810--21823, 2020.

\bibitem{kumar2020conservative}
A.~Kumar, A.~Zhou, G.~Tucker, and S.~Levine, ``Conservative q-learning for
  offline reinforcement learning,'' {\em Advances in Neural Information
  Processing Systems}, vol.~33, pp.~1179--1191, 2020.

\bibitem{prudencio2023survey}
R.~F. Prudencio, M.~R. Maximo, and E.~L. Colombini, ``A survey on offline
  reinforcement learning: Taxonomy, review, and open problems,'' {\em IEEE
  Transactions on Neural Networks and Learning Systems}, 2023.

\bibitem{lu2022challenges}
C.~Lu, P.~J. Ball, T.~G. Rudner, J.~Parker-Holder, M.~A. Osborne, and Y.~W.
  Teh, ``Challenges and opportunities in offline reinforcement learning from
  visual observations,'' {\em arXiv preprint arXiv:2206.04779}, 2022.

\bibitem{abbasifard2014survey}
M.~R. Abbasifard, B.~Ghahremani, and H.~Naderi, ``A survey on nearest neighbor
  search methods,'' {\em International Journal of Computer Applications},
  vol.~95, no.~25, 2014.

\bibitem{zhao2024robust}
P.~Zhao and Z.~Wan, ``Robust nonparametric regression under poisoning attack,''
  in {\em Proceedings of the AAAI Conference on Artificial Intelligence},
  vol.~38, pp.~17007--17015, 2024.

\end{thebibliography}

\newpage

\appendices

\section{Auxiliary Lemmas}\label{sec:lemmas}
This section shows some lemmas that are used in the analysis of both offline and online nearest neighbor $Q$ learning methods. 

The first lemma is about the Lipschitz continuity of $Q^*$, which has been proved \cite{shah2018q}. We prove it again for completeness and consistency of notations.
\begin{lem}\label{lem:lipschitz}
	$Q^*$ is $L$-Lipschitz with respect to $s$, in which
	\begin{eqnarray}
		L=L_r+\gamma C_pQ_m,
	\end{eqnarray}
	with $Q_m:=\sup_{s, a} Q^*(s, a)$ being the maximum $Q^*$.
\end{lem}
\begin{proof}
	Recall the Bellman equation
	\begin{eqnarray}
		Q^*(s,a) = r(s,a) + \gamma \mathbb{E}[\underset{a'}{\max} Q^*(s',a')|s,a].
	\end{eqnarray}
	Denote $Q_m=R/(1-\gamma)$. It can be easily shown that $Q^*(s,a)\leq Q_m$ for all $s\in \mathcal{S}$ and $a\in \mathcal{A}$.
	
	For any $s_1,s_2\in \mathcal{S}$, by Assumption \ref{ass:main} (a) and (c),
	\begin{eqnarray}
		|Q^*(s_2,a) - Q^*(s_1,a)|&\leq & |r(s_2, a) - r(s_1, a)|+\gamma\int (p(s'|s_2,a) - p(s'|s_1,a))\underset{a'}{\max}Q^*(s',a')ds'\nonumber\\
		&\leq & L_r\|s_2-s_1\|+\gamma \int L_p(s')\|s_2-s_1\|\underset{a'}{\max}Q^*(s',a')ds'\nonumber\\
		&\leq & (L_r+\gamma C_pQ_m)\|s_2-s_1\|.
	\end{eqnarray}
	The proof is complete.
\end{proof}

In order to obtain the concentration bounds of the number of steps falling in some fixed region, we prove an extension of Chernoff inequality for sequentially dependent data.
\begin{lem}\label{lem:chernoff}
	Denote $X_{1:i}=(X_1,\ldots, X_i)$ and $x_{1:i}=(x_1,\ldots, x_i)$. Suppose that $X_1\rightarrow \ldots \rightarrow X_n$ form a Markov chain, with $X_i$ be either $0$ or $1$, and $\text{P}(X_{i+1}|X_{1:i}=x_{1:i})\geq p$ for any values of $x_{1:i}$. Then for $k\leq np$, 
	\begin{eqnarray}
		\text{P}\left(\sum_{i=1}^n X_i<k\right) \leq e^{-np}\left(\frac{enp}{k}\right)^k.
	\end{eqnarray}
\end{lem}
\begin{proof}
	The proof just follows the standard proof of Chernoff inequality. The only difference is that the standard Chernoff inequality requires samples to be independent, while now we are analyzing sequentially dependent samples. From the condition $\text{P}(X_{i+1}=1|X_{1:i}=x_{1:i}) \geq p$, for all $\lambda>0$ and any values of $x_{1:i}$,
	\begin{eqnarray}
		\mathbb{E}\left[e^{-\lambda X_{i+1}}|X_{1:i}=x_{1:i}\right]\leq pe^{-\lambda}+1-p.
	\end{eqnarray}
	Therefore
	\begin{eqnarray}
		\mathbb{E}\left[e^{-\lambda \sum_{i=1}^n X_i}\right] &=& \mathbb{E}\left[\mathbb{E}\left[e^{-\lambda \sum_{i=1}^{n-1} X_i} e^{-\lambda X_n}|X_{1:n-1}\right] \right]\nonumber\\
		&\leq & \mathbb{E}\left[ e^{-\lambda \sum_{i=1}^{n-1} X_i} (pe^{-\lambda}+1-p)\right]\nonumber\\
		&\leq & \ldots \nonumber\\
		&\leq & (pe^{-\lambda}+1-p)^n.
	\end{eqnarray}
	Hence
	\begin{eqnarray}
		\text{P}\left(\sum_{i=1}^n X_i\leq k\right) &=& \text{P}\left(-\sum_{i=1}^n X_i\geq -k\right)\nonumber\\
		&=& \inf_\lambda \text{P}\left(e^{-\lambda\sum_{i=1}^n X_i}\geq e^{-\lambda k}\right)\nonumber\\
		&\leq & \inf_\lambda e^{\lambda k}  \mathbb{E}\left[e^{-\lambda \sum_{i=1}^n X_i}\right]\nonumber\\
		&=&\inf_\lambda e^{\lambda k}(pe^{-\lambda} +1-p)^n\nonumber\\
		&=&\exp\left[\inf_\lambda \left[\lambda k + n\ln (pe^{-\lambda}+1-p)\right]\right]\nonumber\\
		&\overset{(a)}{=} & \exp\left[k\ln \frac{p(n-k)}{k(1-p)} + n\ln \frac{n(1-p)}{n-k}\right]\nonumber\\
		&\overset{(b)}{=} & \exp\left[-nD\left(\frac{k}{n}||p\right)\right]\nonumber\\
		&\overset{(c)}{\leq} & e^{-np}\left(\frac{enp}{k}\right)^k.
	\end{eqnarray}
	In (a), we let $\lambda=\ln \frac{p(n-k)}{k(1-p)}$, which takes the minimum over the expression in the previous step. In (b), $D(q||p)=q\ln(q/p)+(1-q)\ln(1-q/(1-p))$ is the Kullback-Leibler (KL) divergence. (c) uses the inequality $D(q||p)\geq p-q-q\ln(p/q)$. The proof is complete.
\end{proof}

The next two lemmas, i.e. Lemma \ref{lem:noise} and Lemma \ref{lem:noise2} provide a uniform bound on the random estimation error for the offline and online $Q$ learning methods respectively.
\begin{lem}\label{lem:noise}
	Define
	\begin{eqnarray}
		U_j = W_j + \gamma \left[\underset{a}{\max}Q^*(S_{j+1}, a) - \mathbb{E}\left[\underset{a}{\max}Q^*(S',a)|S_j, A_j\right]\right],
		\label{eq:uj}
	\end{eqnarray}
	in which $S'$ is a random state generated via $p(\cdot|S_j, A_j)$, $S_{i+1}$ is the actual state at time $i+1$. Furthermore, define
	\begin{eqnarray}
		\sigma_U = \sqrt{\sigma^2+\frac{1}{4}\gamma^2Q_m^2},
		\label{eq:sigmau}
	\end{eqnarray}
	in which $Q_m=\sup_{s,a} Q^*(s,a)$ is the supremum $Q^*$, then for the offline $Q$ learning,
	\begin{eqnarray}
		\text{P}\left(\underset{s\in \mathcal{S}}{\cup}\underset{a\in \mathcal{A}}{\cup} \left\{\left|\frac{1}{k}\sum_{j\in \mathcal{N}(s,a)}U_j\right|>\frac{\sigma_U}{\sqrt{k}}\ln T\right\}\right)\leq dT^{2d}|\mathcal{A}|e^{-\frac{1}{2}\ln^2 T},
		\label{eq:noise1}
	\end{eqnarray}
	in which $\mathcal{N}(s,a)$ is the set of indices of $k$ nearest neighbors of $s$ among all states in the dataset with action $a$, i.e. $\{S_j|A_j= a\}$. 
\end{lem}
\begin{proof}
	The proof uses some ideas from the proof of Lemma 3 in \cite{jiang2019non} and \cite{zhao2024robust}.
	
	In \eqref{eq:uj}, $W_i$ is subgaussian with parameter $\sigma^2$. For the second term in \eqref{eq:uj}, since $Q^*$ is bounded by $R/(1-\gamma)$, conditional on previous state, $\underset{a}{\max}Q^*(S_{j+1}, a) - \mathbb{E}\left[\underset{a}{\max}Q^*(S',a)|S_j, A_j\right]$ is subgaussian with parameter $V_m^2/4$, i.e.
	\begin{eqnarray}
		\mathbb{E}[e^{\lambda U_j}|S_1,A_1,R_1,\ldots, S_{i-1}, A_{i-1}, R_{i-1}, S_i]\leq  \exp\left[\frac{1}{2}\lambda^2\left(\sigma^2+\frac{1}{4}\gamma^2 V_m^2\right)\right]
		=e^{\frac{1}{2}\lambda^2\sigma_U^2},
		\label{eq:umgf}
	\end{eqnarray}
	in which the last step comes from \eqref{eq:sigmau}. Based on \eqref{eq:umgf}, for any fixed set $I\subset \{1,\ldots, T\}$ with $|I|=k$,
	\begin{eqnarray}
		\mathbb{E}\left[\exp\left(\lambda \sum_{j\in I}U_j\right)\right]\leq \exp\left[\frac{k}{2}\lambda^2\sigma_U^2\right],
	\end{eqnarray}
	and
	\begin{eqnarray}
		\text{P}\left(\frac{1}{k}\sum_{j\in I}U_j>t\right)\leq \exp\left[-\frac{kt^2}{2\sigma_U^2}\right].
	\end{eqnarray}
	
	We need to obtain a union bound of $(1/k)\sum_{j\in \mathcal{N}(s,a)} U_j$ that holds with high probability, for all possible sets $\mathcal{N}(s,a)$. Therefore, we need to provide an upper bound of the number of possible datasets $\mathcal{N}(s,a)$. Let $A_{ij}$ be $d-1$ dimensional hyperplane that bisects $S_i$, $S_j$, $0\leq i,j\leq T-1$. The number of planes is at most $N_p=T(T-1)/2$. These hyperplanes divide the state space $\mathcal{S}$ into $N_r$ regions, $N_r$ can be bounded by
	\begin{eqnarray}
		N_r=\sum_{j=0}^d \binom{N_p}{j}\leq dN_p^d\leq dT^{2d}.
	\end{eqnarray} 
	For all $s$ within a region, the $k$ nearest neighbors should be the same. Hence
	\begin{eqnarray}
		|\{\mathcal{N}(s,a)|s\in \mathcal{S}, a\in \mathcal{A} \}|\leq dT^{2d} |\mathcal{A}|.
		\label{eq:nsets}
	\end{eqnarray}
	 
	Combining with \eqref{eq:nsets}, and taking union for all possible sets $\mathcal{N}_t(s,a)$, as well as all $t$, we have
	\begin{eqnarray}
		\text{P}\left(\underset{s\in \mathcal{S}}{\cup}\underset{a\in \mathcal{A}}{\cup} \left\{\left|\frac{1}{k}\sum_{j\in \mathcal{N}(s,a)}U_j\right|>u                  \right\}\right)\leq dT^{2d}|\mathcal{A}|e^{-\frac{ku^2}{2\sigma_U^2}}.
		\label{eq:noise1t}
	\end{eqnarray}	
	Let $u=\sigma_U\ln T/\sqrt{k}$, the proof of \eqref{eq:noise1} is complete.
\end{proof}
\begin{lem}\label{lem:noise2}
	For the online method,
	\begin{eqnarray}
		\text{P}\left(\underset{s\in \mathcal{S}}{\cup}\underset{a\in \mathcal{A}}{\cup}\underset{t\leq T}{\cup} \left\{\left|\frac{1}{k(t)}\sum_{j\in \mathcal{N}_t(s,a)}U_j\right|>\frac{\sigma_U}{\sqrt{k(t)}}\ln T\right\}\right)\leq d(1-\beta)^{2d}T^{2d+1}|\mathcal{A}|e^{-\frac{1}{2}\ln^2 T},
		\label{eq:noise2}
	\end{eqnarray}
	in which $\mathcal{N}_t(s,a)$ is the set of $k(t)$ nearest neighbors of $s$ among $\{S_j|\beta\leq t, A_j=a\}$.	
\end{lem}
\begin{proof}
	The proof of Lemma \ref{lem:noise2} is only slightly different from the proof of Lemma \ref{lem:noise}. We still let $A_{ij}$ be $d-1$ dimensional hyperplane that bisects $S_i$, $S_j$, but now the range of $i$, $j$ becomes $\beta t\leq i,j< t$. The number of planes is at most $N_p=N(N-1)/2$, in which $N\leq (1-\beta) t$. Then the number of regions $N_r$ becomes
	\begin{eqnarray}
		N_r=\sum_{j=0}^d \binom{N_p}{j}\leq dN_p^d\leq dN^{2d}\leq d(1-\beta)^{2d} t^{2d}.
	\end{eqnarray} 
	For all $s$ within a region, the $k$ nearest neighbors should be the same. Hence
	\begin{eqnarray}
		|\{\mathcal{N}_t(s,a)|s\in \mathcal{S}, a\in \mathcal{A} \}|\leq d(1-\beta)^{2d} t^{2d} |\mathcal{A}|.
		\label{eq:nsets2}
	\end{eqnarray}
	Compared with \eqref{eq:nsets}, there is an additional $(1-\beta)^{2d}$ factor. Other steps are the same as the proof of \eqref{eq:noise1}. The result is
	\begin{eqnarray}
		\text{P}\left(\cup_{s\in \mathcal{S}} \cup_{a\in \mathcal{A}} \left\{\left|\frac{1}{k(t)}\sum_{j\in \mathcal{N}_t(s,a)} U_j\right|> u \right\}\right) \leq d(1-\beta)^{2d} T^{2d} |\mathcal{A}| e^{-\frac{k(t) u^2}{2\sigma_U^2}}.
		\label{eq:onlineub}
	\end{eqnarray}
	Let $u=\sigma_U\ln T/\sqrt{k(t)}$, and take union bound over $t=1,\ldots, T$, \eqref{eq:onlineub} becomes
	\begin{eqnarray}
		\text{P}\left(\underset{s\in \mathcal{S}}{\cup}\underset{a\in \mathcal{A}}{\cup}\underset{t\leq T}{\cup} \left\{\left|\frac{1}{k(t)}\sum_{j\in \mathcal{N}_t(s,a)}U_j\right|>\frac{\sigma_U}{\sqrt{k(t)}}\ln T\right\}\right)\leq d(1-\beta)^{2d}T^{2d+1}|\mathcal{A}|e^{-\frac{1}{2}\ln^2 T}.
	\end{eqnarray}
\end{proof}

The next two lemmas, i.e. Lemma \ref{lem:rho} and Lemma \ref{lem:rho2} bound the $k$ nearest neighbor distances for the offline and online $Q$ learning methods respectively.
\begin{lem}\label{lem:rho}
	Define
	\begin{eqnarray}
		\rho_0(s,a)&=&\underset{j\in \mathcal{N}(s,a)}{\max}\norm{S_j-s},
		\label{eq:rho0}\\
		r_0&=&\left(\frac{3km}{\pi_0c\alpha v_d T}\right)^\frac{1}{d},	
		\label{eq:r0}		
	\end{eqnarray}
	in which $m, \pi_0, c,\alpha$ are constants in Assumptions \ref{ass:main} and \ref{ass:bounded}. Then for the offline method, if $T\geq 3m$, then
	\begin{eqnarray}
		\text{P}\left(\underset{s\in \mathcal{S}}{\cup}\underset{a\in \mathcal{A}}{\cup}\{\rho_0(s,a)>2r_0 \} \right)\leq  \left(\frac{\pi_0c\alpha v_dC_ST}{2km}+1\right)|\mathcal{A}|e^{-(1-\ln 2)k}.	
		\label{eq:rho1}
	\end{eqnarray}
\end{lem}
\begin{proof}
	Define
	\begin{eqnarray}
		n(s,a, r)=\sum_{t=1}^T \mathbf{1}\left(\norm{S_t-s}\leq r, A_t=a\right).
		\label{eq:ndf}
	\end{eqnarray}
	Then
	\begin{eqnarray}
		\text{P}(\rho_0(s, a)>r_0)\leq \text{P}(n(s, a, r_0)<k).
		\label{eq:conv}
	\end{eqnarray}
	It remains to bound $\text{P}(n(s, a, r_0)<k)$. According to Assumption \ref{ass:bounded}(e), for all $s$,
	
	\begin{eqnarray}
		\text{P}\left(\norm{S_{t+m}-s}\leq r_0|S_t, A_t\right) &\overset{(a)}{=} & \int_{B(s, r_0)} p_\pi^m(u|S_{t}, A_t)du\nonumber\\
		&\overset{(b)}{\geq} & cV(B(s, r_0)\cap \mathcal{S})\nonumber\\
		&\overset{(c)}{\geq} & c\alpha v_d r_0^d.
		\label{eq:smass}
	\end{eqnarray}
	For (a), recall Assumption \ref{ass:bounded}(e), $p_\pi^m$ is the $m$ step transition kernel. (b) holds since $p_\pi^m(y|s, a)\geq c$ always hold. (c) comes from Assumption \ref{ass:bounded}(f). Moreover, by Assumption \ref{ass:main}(d),
	\begin{eqnarray}
		\text{P}\left(\norm{S_{t+m}-s}\leq r_0, A_{t+m} = a|S_t, A_t\right) \geq \pi_0c\alpha v_dr_0^d=\frac{3km}{T}.
		\label{eq:plba}
	\end{eqnarray}
	Now we use Lemma \ref{lem:chernoff} to bound $\text{P}(n(s, a, r_0)<k)$. Let
	\begin{eqnarray}
		X_i=\mathbf{1}\left(\norm{S_{i\cdot m}-s}\leq r_0, A_{i\cdot m}= a\right),
	\end{eqnarray}
	for $i=1,\ldots, \lfloor T/m\rfloor$. Then the conditions in Lemma \ref{lem:chernoff} are satisfied with $p=3km/T$. Hence as long as $T\geq 3m$ holds,
	\begin{eqnarray}
		\text{P}(n(s, a, r_0)<k) &\leq & \text{P}\left(\sum_{i=1}^{\lfloor T/m\rfloor} X_i<k\right)\nonumber\\
		&\leq & e^{-\lfloor T/m\rfloor \frac{3km}{T}}\left(\frac{e\left\lfloor \frac{T}{m}\right\rfloor \frac{3km}{T}}{k}\right)^k\nonumber\\
		&\overset{(a)}{\leq} & e^{-2k} (2e)^k\nonumber\\
		&=& e^{-(1-\ln 2)k},
		\label{eq:smalln}
	\end{eqnarray}
	in which (a) holds because
	\begin{eqnarray}
		\left\lfloor \frac{T}{m}\right\rfloor \frac{3km}{T}\geq \left(\frac{T}{m}-1\right) \frac{3km}{T}=3k\left(1-\frac{m}{T}\right) \geq 2k.
	\end{eqnarray}
	From \eqref{eq:conv}, $\text{P}(\rho_0(s, a)>r_0) \leq e^{-(1-\ln 2)k}$.
	Now it remains to obtain a uniform upper bound over all $s\in \mathcal{S}$ and $a\in \mathcal{A}$. Find a $r_0$ covering of $\mathcal{S}$: $G_1,\ldots, G_{n_c}$, such that for all $s\in \mathcal{S}$, there exists $i$ such that $\|s-G_i\|\leq r_0$. From Assumption \ref{ass:main}(g),
	\begin{eqnarray}
		n_c\leq \frac{C_S}{r_0^d} + 1=\frac{\pi_0c\alpha v_dC_ST}{2km}+1.
	\end{eqnarray}
	Then
	\begin{eqnarray}
		\text{P}\left(\underset{s\in \mathcal{S}}{\cup}\underset{a\in \mathcal{A}}{\cup}\{\rho_0(s,a)>2r_0 \} \right)&\leq& \text{P}\left(\exists i\in [n_c],\rho_0(S_j, a)>r_0\right)\nonumber\\
		&\leq & n_c|\mathcal{A}|e^{-(1-\ln 2)k}.
	\end{eqnarray}	
\end{proof}
\begin{lem}\label{lem:rho2}
	Define
	\begin{eqnarray}
		\rho_t(s,a)&=&\underset{j\in \mathcal{N}_t(s,a)}{\max}\|S_j-s\|,
		\label{eq:rhot}\\
		r_t&=&\left(\frac{3km}{(1-\beta)\pi_0c\alpha v_d t}\right)^\frac{1}{d},
		\label{eq:rt}\\
		t_c&=&\max\left\{\frac{3m}{1-\beta}, (\ln^2 T+1)^{\frac{d+2}{2}}\right\}.
		\label{eq:tc}	
	\end{eqnarray}
	Then for the online method, we have
	\begin{eqnarray}
		\text{P}\left(\underset{s\in \mathcal{S}}{\cup}\underset{a\in \mathcal{A}}{\cup}\underset{t_c \leq t\leq T}{\cup} \{\rho_t(s,a)> 2r_t\}\right)\leq \left[\frac{(1-\beta)\pi_0c\alpha v_dC_S t}{3km} + 1\right]T|\mathcal{A}| e^{-(1-\ln 2)\ln^2 T}.
		\label{eq:rho2}
	\end{eqnarray}	
\end{lem}
\begin{proof}
	We only show the difference with the proof of Lemma \ref{lem:rho}. Other steps are similar and hence are omitted. Define
	\begin{eqnarray}
		n_t(s, a, r)=\sum_{j=\lceil \beta t\rceil }^{t-1}\mathbf{1}(\norm{S_j-s}\leq r, A_t=a).
		\label{eq:ntdf}
	\end{eqnarray}
	Then \eqref{eq:plba} becomes
	\begin{eqnarray}
		\text{P}\left(\norm{S_{t+m}-s}\leq r_t, A_{t+m}=a|S_t, A_t\right)\geq \frac{3km}{(1-\beta) t}.
		\label{eq:pmass}
	\end{eqnarray}
	Now let
	\begin{eqnarray}
		X_i=\mathbf{1}\left(\norm{S_{\lceil \beta t\rceil + i\cdot m}-s}\leq r_t, A_{\lceil \beta t\rceil+i\cdot m} = a\right),
	\end{eqnarray}
	for $i=1,\ldots, \lfloor (1-\beta)t/m\rfloor$. Then the conditions in Lemma \ref{lem:chernoff} are satisfied with $p=3km/((1-\beta) t)$. Hence for $t\geq t_c$,
	\begin{eqnarray}
		\text{P}(n_t(s, a, r_t)<k)&\leq & \text{P}\left(\sum_{i=1}^{\lceil (1-\beta) t/m\rceil} X_i<k\right)\nonumber\\
		&\leq & \exp\left[-\left\lfloor \frac{(1-\beta) t}{m}\right\rfloor \frac{3km}{(1-\beta) t}\right] \left(\frac{e\left\lfloor \frac{(1-\beta) t}{m}\right\rfloor \frac{3km}{(1-\beta) t}}{k}\right)^k\nonumber\\
		&\overset{(\star)}{\leq} & e^{-2k}(2e)^k\nonumber\\
		&=& e^{-(1-\ln 2)k},
		\label{eq:ntlarge}
	\end{eqnarray}
	in which $(\star)$ holds since for $t\geq t_c$,
	\begin{eqnarray}
		\left\lfloor \frac{(1-\beta) t}{m}\right\rfloor \frac{3km}{(1-\beta) t}\geq 3k\left(1-\frac{m}{(1-\beta) t}\right) \geq 3k\left(1-\frac{m}{(1-\beta) t_c}\right)\geq 2k.
	\end{eqnarray}
	Similar to \eqref{eq:conv}, $\text{P}(\rho_t(s, a)>r_t)=\text{P}(n_t(s, a, r_t)<k)$. Therefore
	\begin{eqnarray}
		\text{P}(\rho_t(s, a)>r_t)\leq e^{-(1-\ln 2)k}.
	\end{eqnarray}
	From \eqref{eq:tc}, if $t\geq t_c$, then $k=\lfloor t^{2/(d+2)}\rfloor\geq \ln^2 T$. Therefore $\text{P}(\rho_t(s,a)>r_t)\leq e^{-(1-\ln 2)\ln^2 T}$. Now we find a $r_t$ covering of $\mathcal{S}$ with cover number $n_{ct}$. For any fixed $t$,
	\begin{eqnarray}
		\text{P}\left(\underset{s\in \mathcal{S}}{\cup}\underset{a\in \mathcal{A}}{\cup}\{\rho_t(s,a)>2r_t \} \right)&\leq& n_{ct}|\mathcal{A}|e^{-(1-\ln 2)k}\nonumber\\	
		&\leq & \left(\frac{(1-\beta)\pi_0c\alpha v_dC_St}{3km} + 1\right)|\mathcal{A}|e^{-(1-\ln 2)\ln^2 T}.	
	\end{eqnarray}
	Taking union bound over all $t$, \eqref{eq:rho2} can be proved.
\end{proof}
\section{Proof of Theorem \ref{thm:offline}}\label{sec:offlinepf}
This section focuses on the error bound of the offline method. We begin with the following lemma.
\begin{lem}\label{lem:fix}
	After infinite number of iterations, $q$ and $Q$ satisfy
	\begin{eqnarray}
		Q(t)&=&R_t+\gamma \underset{a}{\max}q(S_{t+1}, a), \label{eq:relation1}\\
		q(s,a)&=&\frac{1}{k}\sum_{j\in \mathcal{N}(s,a)}Q(j).
		\label{eq:relation2}
	\end{eqnarray}
\end{lem}

\begin{proof}
	Recall \eqref{eq:step1} and \eqref{eq:step2}. $Q_{i}(t)$ and $q_i(s,a)$ are the values of $Q(t)$ and $q(s,a)$ at the $i$-th iteration, respectively. Then
	\begin{eqnarray}
		Q_{i+1}(t)&=&R_t+\gamma \underset{a'}{\max} q_i(S_{t+1},a'); \label{eq:step1app}\\
		q_{i+1}(S_t,a)&=&\frac{1}{k}\sum_{j\in \mathcal{N}(s,a)} Q_{i+1}(j).
		\label{eq:step2app}
	\end{eqnarray}
	From \eqref{eq:step1app} and \eqref{eq:step2app},
	\begin{eqnarray}
		Q_{i+1}(t)=R_t+\gamma \underset{a'}{\max}\frac{1}{k}\sum_{j\in \mathcal{N}(S_{t+1}, a')}Q_{i}(j).
	\end{eqnarray}
	Define an operator $F$ such that
	\begin{eqnarray}
		F[Q_i](t) = R_t+\gamma \underset{a'}{\max}\frac{1}{k}\sum_{j\in \mathcal{N}(S_{t+1}, a')}Q_{i}(j),
		\label{eq:F}
	\end{eqnarray}
	and
	\begin{eqnarray}
		\norm{Q_i-Q_i'}_\infty=\underset{t=1,\ldots, T}{\max}|Q_i(t)-Q_i'(t)|.
	\end{eqnarray}
	Then
	\begin{eqnarray}
		\norm{F[Q_i]-F[Q_i]'}_\infty\leq \gamma\norm{Q_i-Q_i'}_\infty.
	\end{eqnarray}
	Since $0<\gamma<1$, according to Banach fixed point theorem, there exists a $Q$ function such that
	\begin{eqnarray}
		Q(t) = F[Q](t), t=1,\ldots,T,
		\label{eq:fixed}
	\end{eqnarray}
	and $\underset{i\rightarrow\infty}{\lim}\norm{Q_i-Q}=0$. From \eqref{eq:F}, with the limit of $i\rightarrow \infty$, using \eqref{eq:fixed}, we have
	\begin{eqnarray}
		Q(t)=R_t+\gamma \underset{a'}{\max}\frac{1}{k}\sum_{j\in \mathcal{N}(S_{t+1}, a')} Q(j).
	\end{eqnarray}
	Moreover, note that
	\begin{eqnarray}
		q(s,a) = \underset{i\rightarrow\infty}{\lim} q_i(s,a) = \frac{1}{k}\sum_{j\in \mathcal{N}(s,a)}\underset{i\rightarrow\infty}{\lim}Q_{i}(j) = \frac{1}{k}\sum_{j\in \mathcal{N}(s,a)} Q(j).
		\label{eq:ansa}
	\end{eqnarray}
	Therefore
	\begin{eqnarray}
		Q(t)=R_t+\gamma\underset{a'}{\max}q(S_{t+1}, a').
		\label{eq:ansb}
	\end{eqnarray}
	\eqref{eq:ansa} and \eqref{eq:ansb} are exactly the conclusion of Lemma \ref{lem:fix}. The proof is complete.
\end{proof}
With Lemma \ref{lem:fix}, it remains to bound the estimation error. Let $S'$ be a random state following distribution $p(\cdot|S_t,A_t)$. Then from \eqref{eq:relation1},
\begin{eqnarray}
	Q(t)-Q^*(S_t, A_t)&=& R_t+\gamma \underset{a}{\max}q(S_{t+1}, a)-Q^*(S_t, A_t)\nonumber\\
	&\overset{(a)}{=} & R_t+\gamma \underset{a}{\max}q(S_{t+1}, a) - r(S_t, A_t)-\gamma \mathbb{E}\left[\underset{a}{\max}Q^*(S', a)|S_t, A_t\right]\nonumber\\
	&\overset{(b)}{=} & W_t+\gamma \underset{a}{\max}Q^*(S_{t+1}, a)-\gamma \mathbb{E}\left[\underset{a}{\max}Q^*(S', a)|S_t, A_t\right]\nonumber\\
	&&+\gamma \underset{a}{\max}q(S_{t+1}, a) - \gamma \underset{a}{\max}Q^*(S_{t+1}, a)\nonumber\\
	&\overset{(c)}{=} & U_t+\gamma \underset{a}{\max}q(S_{t+1}, a)-\gamma \underset{a}{\max}Q^*(S_{t+1}, a),
\end{eqnarray}
in which (a) comes from the Bellman equation \eqref{eq:bellman}, (b) comes from \eqref{eq:R}, and (c) comes from \eqref{eq:uj}. From \eqref{eq:relation2},
\begin{eqnarray}
	q(s,a)-Q^*(s,a)&=&\frac{1}{k}\sum_{j\in \mathcal{N}(s,a)}(Q(j)-Q^*(s,a))\nonumber\\
	&=& \frac{1}{k}\sum_{j\in \mathcal{N}(s,a)} (Q(j)-Q^*(S_j, A_j))+\frac{1}{k}\sum_{j\in \mathcal{N}(s,a)} (Q^*(S_j, A_j) - Q^*(s,a))\nonumber\\
	&=&\frac{1}{k}\sum_{j\in \mathcal{N}(s,a)}\left[U_j+\gamma \underset{a'}{\max}q(S_{j+1}, a')-\gamma \underset{a'}{\max} Q^*(S_{j+1}, a')\right]\nonumber\\
	&&+\frac{1}{k}\sum_{j\in \mathcal{N}(s,a)} (Q^*(S_j, A_j) - Q^*(s,a)).
	\label{eq:qerr}
\end{eqnarray}
Therefore, from Lemma \ref{lem:lipschitz},
\begin{eqnarray}
	|q(s,a) - Q^*(s,a)|
	\leq \gamma \left|\frac{1}{k}\sum_{j\in \mathcal{N}(s,a)} \left(\underset{a'}{\max}q(S_{j+1}, a')-\underset{a'}{\max}Q^*(S_{j+1}, a')\right)\right|+\left|\frac{1}{k}\sum_{j\in \mathcal{N}(s,a)}U_j\right|+L\rho_0(s,a),\hspace{-5mm}\nonumber\\
	\label{eq:qsplit}
\end{eqnarray}
in which $\rho_0$ has been defined in \eqref{eq:rho0}. Define the estimation error
\begin{eqnarray}
	\epsilon := \norm{q-Q^*}_\infty =\underset{s,a}{\sup} |q(s,a)-Q^*(s,a)|.
	\label{eq:epsdef}
\end{eqnarray}
Then
\begin{eqnarray}
	\epsilon\leq \gamma\epsilon+\underset{s,a}{\sup}\left|\frac{1}{k}\sum_{j\in \mathcal{N}(s,a)}U_j\right|+L\underset{s,a}{\sup}\rho_0(s,a),
\end{eqnarray}
i.e.
\begin{eqnarray}
	\epsilon\leq \frac{1}{1-\gamma}\left[\underset{s,a}{\sup}\left|\frac{1}{k}\sum_{j\in \mathcal{N}(s,a)}U_j\right|+L\underset{s,a}{\sup}\rho_0(s,a)\right].
	\label{eq:eps}
\end{eqnarray}

From \eqref{eq:noise1}, \eqref{eq:rho1} from Lemmas \ref{lem:noise} and \ref{lem:rho}, we have that, with probability at least $1-\delta$, in which
\begin{eqnarray}
	\delta = dT^{2d}|\mathcal{A}|e^{-\frac{1}{2}\ln^2 T}+\left(\frac{\pi_0c\alpha v_dC_ST}{2km}+1\right)|\mathcal{A}|e^{-(1-\ln 2)k},
	\label{eq:prob}
\end{eqnarray}
the following two equations hold:
\begin{eqnarray}
	\underset{s\in \mathcal{S}}{\sup}\underset{a\in \mathcal{A}}{\sup}\left|\frac{1}{k}\sum_{j\in \mathcal{N}(s,a)} U_j\right|\leq \frac{\sigma_U}{\sqrt{k}}\ln T,
	\label{eq:eps1}
\end{eqnarray}
in which $\sigma_U$ is defined in \eqref{eq:sigmau}, and
\begin{eqnarray}
	\underset{s\in \mathcal{S}}{\sup}\underset{a\in \mathcal{A}}{\sup} \rho_0(s,a)\leq 2r_0=2\left(\frac{3km}{\pi_0c\alpha v_d T}\right)^\frac{1}{d}.
	\label{eq:eps2}
\end{eqnarray}
From \eqref{eq:eps}, \eqref{eq:eps1} and \eqref{eq:eps2}, we have the following asymptotic bound that holds with probability at least $1-\delta$:
\begin{eqnarray}
	\epsilon\lesssim \frac{1}{1-\gamma}\left( \frac{\ln T}{\sqrt{k}}+\left(\frac{k}{T}\right)^\frac{1}{d}\right).
	\label{eq:asymptotic}
\end{eqnarray}
Now it remains to tune $k$ to minimize the right hand side of \eqref{eq:asymptotic}. The best rate of growth of $k$ with respect to $T$ is
\begin{eqnarray}
	k \sim T^\frac{2}{d+2}.
	\label{eq:kgrow}
\end{eqnarray}
Then with probability $1-\delta$, in which $\delta$ is defined in \eqref{eq:prob},
\begin{eqnarray}
	\epsilon\lesssim \frac{1}{1-\gamma}T^{-\frac{1}{d+2}}\ln T.
\end{eqnarray}
Therefore the sample complexity is
\begin{eqnarray}
	T=\tilde{O}\left(\frac{1}{(1-\gamma)^{d+2}\epsilon^{d+2}}\right).
\end{eqnarray}
The proof of Theorem 1 is complete.

\section{Proof of Theorem \ref{thm:tail}}\label{sec:tail}
We begin with the following lemmas.
\begin{lem}\label{lem:largeu}
	\begin{eqnarray}
		\mathbb{E}\left[\underset{s, a}{\sup}\left|\frac{1}{k}\sum_{j\in \mathcal{N}(s,a)} U_j\right|\right]\leq \sqrt{\frac{2\sigma_U^2}{k}\ln(dT^{2d}|\mathcal{A}|)}+\sqrt{\frac{2\pi \sigma_U^2}{k}}.
	\end{eqnarray}
\end{lem}
The proof of Lemma \ref{lem:largeu} is shown in Appendix \ref{sec:largeu}. The next lemma gives a bound of the expectation of kNN radius of $s$, which depends on $g(s)$ defined in \eqref{eq:g}.
\begin{lem}\label{lem:rhotail}
	If $g(s)\geq 3mk/(\pi_0\alpha v_dD^dT)$, then for some constant $C_1$,
	\begin{eqnarray}
		\mathbb{E}\left[\underset{a}{\max}\rho_0(s, a)\right]\leq 
			\left(\frac{3mk}{\pi_0\alpha v_dTg(s)}\right)^\frac{1}{d} + C_1(\norm{s}+1)|\mathcal{A}|e^{-(1-\ln 2)k}.
	\end{eqnarray}
	Otherwise, for some constant $C_2$,
	\begin{eqnarray}
		\mathbb{E}\left[\underset{a}{\max}\rho_0(s, a)\right]\leq C_2(\norm{s}+1).
	\end{eqnarray}
\end{lem}

The proof of Lemma \ref{lem:rhotail} is shown in Appendix \ref{sec:rhotail}. Based on Lemma \ref{lem:rhotail}, we then show the following lemma.
\begin{lem}\label{lem:rhotail2}
	There exists a constant $C_3$, such that
	\begin{eqnarray}
		\mathbb{E}\left[\max_{a'} \rho_0(S', a')|s, a\right] \leq C_3\left(\frac{k}{T}\right)^\frac{1}{d},
	\end{eqnarray}
	in which $S'\sim p(\cdot|s, a)$.
\end{lem}
Lemma \ref{lem:rhotail2} indicates that under Assumption \ref{ass:tail}, given the current state $s$, the expectation of kNN distances of next state $S'$ is still bounded by $O((k/T)^{1/d})$, which is the same as the case with bounded support.

With the preparations above, we then bound the estimation error of $Q^*$. Recall \eqref{eq:qsplit}, which bounds the estimation error $|q(s, a)-Q^*(s, a)|$. Intuitively, it is unlikely to obtain a uniform bound, since $S_{j+1}$ may fall at the tail of the support $\mathcal{S}$, thus $|q(S_{j+1}, a)-Q^*(S_{j+1}, a)|$ may be large. Therefore, instead of uniform bound, we bound the expectation of $\ell_1$ error here. Define
\begin{eqnarray}
	\Delta(s):=\max_a \left[|q(s, a)-Q^*(s, a)|-\left|\frac{1}{k}\sum_{j\in \mathcal{N}(s,a)}U_j\right|-L\rho_0(s, a)\right].
	\label{eq:deltas}
\end{eqnarray}
Then for all $a$,
\begin{eqnarray}
	|q(s, a)-Q^*(s, a)|\leq \Delta(s)+\left|\frac{1}{k}\sum_{j\in \mathcal{N}(s,a)}U_j\right|+L\rho_0(s, a).
	\label{eq:qbound}
\end{eqnarray}
From \eqref{eq:qsplit}, \eqref{eq:deltas} and \eqref{eq:qbound},
\begin{eqnarray}
	\Delta(s)&\leq& \frac{\gamma}{k}\max_a\left|\sum_{j\in \mathcal{N}(s,a)}\left[\max_{a'}q(S_{j+1}, a')-\max_{a'} Q^*(S_{j+1}, a')\right]\right|\nonumber\\
	&\leq & \frac{\gamma}{k}\max_{a} \sum_{j\in \mathcal{N}(s, a)} \max_{a'} \left|q(S_{j+1}, a')-Q^*(S_{j+1}, a')\right|\nonumber\\
	&\leq & \frac{\gamma}{k}\max_a \sum_{j\in \mathcal{N}(s,a)}\max_{a'}\left[\Delta(S_{j+1})+\left|\frac{1}{k}\sum_{l\in \mathcal{N}(S_{j+1}, a')} U_l\right|+L\max_{a'} \rho_0(S_{t+1}, a')\right].
	\label{eq:deltatrans}
\end{eqnarray}
Define
\begin{eqnarray}
	\Delta_0:=\mathbb{E}\left[\max_s \Delta(s)\right],
	\label{eq:delta0}
\end{eqnarray}
then from \eqref{eq:deltatrans},
\begin{eqnarray}
	\Delta_0\leq \frac{\gamma}{1-\gamma}\left[\mathbb{E}\left[\underset{s, a}{\sup}\left|\frac{1}{k}\sum_{j\in \mathcal{N}(s, a)} U_j\right|\right]+L\underset{s, a}{\sup}\mathbb{E}\left[\max_{a'} \rho_0(s', a')|s, a\right]\right].
\end{eqnarray}
From Lemmas \ref{lem:largeu} and \ref{lem:rhotail2},
\begin{eqnarray}
	\Delta_0&\leq& \frac{\gamma}{1-\gamma}\left[\sqrt{\frac{2\sigma_U^2}{k}\ln (dT^{2d} |\mathcal{A}|)}+\sqrt{\frac{2\pi \sigma_U^2}{k}}+C_3\left(\frac{k}{T}\right)^\frac{1}{d}\right]\nonumber\\
	&\lesssim & \frac{1}{1-\gamma}\left[\sqrt{\frac{1}{k}\ln T}+\left(\frac{k}{T}\right)^\frac{1}{d}\right].
\end{eqnarray}
Let $k\sim T^{2/(d+2)}$, then
\begin{eqnarray}
	\Delta_0\lesssim \frac{1}{1-\gamma} T^{-\frac{1}{d+2}}\sqrt{\ln T}.
\end{eqnarray}
Recall the definition of $\Delta_0$ in \eqref{eq:delta0}, and the definition of $\Delta(s)$ in \eqref{eq:deltas}, with Lemma \ref{lem:largeu} and Lemma \ref{lem:rhotail},
\begin{eqnarray}
	\mathbb{E}[|q(s, a)-Q^*(s, a)|]&\lesssim& \frac{1}{1-\gamma}T^{-\frac{1}{d+2}}\sqrt{\ln T}+\mathbb{E}\left[\left|\frac{1}{k}\sum_{j\in \mathcal{N}(s,a)} U_j\right|\right]+L\mathbb{E}[\rho_0(s, a)]\nonumber\\
	&\lesssim & \frac{1}{1-\gamma}T^{-\frac{1}{d+2}}\sqrt{\ln T} + \phi(s),
	\label{eq:l1err}
\end{eqnarray}
in which
\begin{eqnarray}
	\phi(s)\lesssim \left\{
	\begin{array}{ccc}
		T^{-\frac{1}{d+2}} g^{-\frac{1}{d}}(s)+(\norm{s}+1) e^{-(1-\ln 2)k} &\text{if} & g(s)\geq \frac{3mk}{\pi_0\alpha v_dD^d T}\\
		\norm{s}+1 &\text{if} &g(s)< \frac{3mk}{\pi_0\alpha v_dD^d T}.
	\end{array}
	\right.
\end{eqnarray}
Taking integration over $\phi(s)$ weighted by the stationary distribution $f_\pi(s)$ yields
\begin{eqnarray}
	\phi(s) f_\pi(s) ds&\lesssim& \int \left[T^{-\frac{1}{d+2}} g^{-\frac{1}{d}}(s)+(\norm{s}+1)e^{-(1-\ln 2)k}\right] f_\pi(s) ds\nonumber\\
	&& + \int (\norm{s}+1)\mathbf{1}\left(g(s)<\frac{3mk}{\pi_0\alpha v_dD^d T}\right) f_\pi(s) ds\nonumber\\
	&\lesssim & T^{-\frac{1}{d+2}}+e^{-(1-\ln 2)k} + \left(\frac{k}{T}\right)^\frac{1}{d}\nonumber\\
	&\sim & T^{-\frac{1}{d+2}},
\end{eqnarray}
in which the second step uses Assumption \ref{ass:tail}(e'). Therefore
\begin{eqnarray}
	\int \mathbb{E}\left[\max_a|q(s, a)-Q^*(s,a)|\right] f_\pi(s)ds\lesssim \frac{1}{1-\gamma}T^{-\frac{1}{d+2}}\ln T.
\end{eqnarray}
The proof is complete.

\subsection{Proof of Lemma \ref{lem:largeu}}\label{sec:largeu}
The proof is based on \eqref{eq:noise1t}. Define
\begin{eqnarray}
	t_0=\sqrt{\frac{2\sigma_U^2}{k}\ln(dT^{2d} |\mathcal{A}|)}.
\end{eqnarray}
Then
\begin{eqnarray}
	\mathbb{E}\left[\underset{s, a}{\sup}\left|\frac{1}{k}\sum_{j\in \mathcal{N}(s, a)} U_j\right|\right] &=&\int_0^\infty \text{P}\left(\underset{s, a}{\sup}\left|\frac{1}{k}\sum_{j\in \mathcal{N}(s,a)} U_j\right|>t\right) dt\nonumber\\
	&\leq & \int_0^{t_0}1 dt+\int_{t_0}^\infty dT^{2d} |\mathcal{A}|e^{-\frac{kt^2}{2\sigma_U^2}} dt\nonumber\\
	&\overset{(a)}{\leq} & t_0+\frac{\sigma_UdT^{2d}|\mathcal{A}|}{\sqrt{k}}\sqrt{2\pi} e^{-\ln (dT^{2d}|\mathcal{A}|)}\nonumber\\
	&=&\sqrt{\frac{2\sigma_U^2}{k} \ln(dT^{2d}|\mathcal{A}|)}+\sqrt{\frac{2\pi\sigma_U^2}{k}}.
\end{eqnarray}
in which (a) uses the inequality $\int_t^\infty e^{-x^2/2}dx\leq \sqrt{2\pi}e^{-t^2/2}$. The proof is complete.

\subsection{Proof of Lemma \ref{lem:rhotail}}\label{sec:rhotail}
The beginning of our proof follows that of Lemma \ref{lem:rho}. The difference is that now the support is unbounded, thus the density is no longer bounded away from zero.
 
For $r\leq D$, $t=m+1, 2m+1,\ldots$, recall the definition of $g$ in \eqref{eq:g},
\begin{eqnarray}
	\text{P}(\norm{S_t-s}\leq r, A_t=a|S_1,A_1,R_1,\ldots, S_{t-1}, A_{t-1}, R_{t-1})&\geq& \pi_0\int_{B(s, r)} g(u) du\nonumber\\
	&\geq & \pi_0\alpha v_d r^d g(s),
\end{eqnarray}
in which the second step comes from Assumption \ref{ass:tail}(f'). Define
\begin{eqnarray}
	r_0(s)=\left(\frac{3mk}{\pi_0\alpha v_dTg(s)}\right)^\frac{1}{d}.
\end{eqnarray}
Now we discuss the following two cases separately.

\textbf{Case 1: $r_0(s)\leq D$.} Recall the definition of $n(s, a, r)$ in \eqref{eq:ndf}. According to Assumption \ref{ass:tail}(e'), similar to \eqref{eq:smass},
\begin{eqnarray}
	\text{P}\left(\norm{S_{t+m}-s}\leq r_0(s)\right) &\overset{(a)}{\geq} &\int_{B(s, r_0(s)}g(u)du\nonumber\\
	&\overset{(b)}{\geq}&\alpha v_dr_0^d(s) g(s),
\end{eqnarray}
in which (a) comes from the definition of function $g$ in \eqref{eq:g}, and (b) comes from Assumption \ref{ass:tail}(f'). Therefore
\begin{eqnarray}
	\text{P}\left(\norm{S_{t+m}-s}\leq r_0(s), A_t=a|S_t, A_t\right)\geq \pi_0 \alpha v_dr_0^d(s) g(s) = \frac{3mk}{T}.
\end{eqnarray}

Following the arguments of \eqref{eq:smalln}, 
\begin{eqnarray}
	\text{P}(n(s, a, r_0(s))<k)\leq e^{-(1-\ln 2)k}.
\end{eqnarray}
Hence
\begin{eqnarray}
	\text{P}(\rho_0(s,a)>r_0(s))\leq e^{-(1-\ln 2)k}.
	\label{eq:r0bound}
\end{eqnarray}
$r_0(s)$ is a high probability upper bound of $\rho_0(s, a)$. To bound $\mathbb{E}[\rho_0(s, a)]$, it is necessary to bound $\text{P}(\rho_0(s,a)>r)$ for large $r$. From Assumption \ref{ass:tail}(g'), for $S'\sim p(\cdot|s, a)$,
\begin{eqnarray}
	\text{P}(\norm{S'}>r|s, a)\leq \frac{C_0}{r}.
\end{eqnarray}
Denote $S_{1:t-1}=(S_1,\ldots, S_{t-1})$, and $A_{1:t-1}, R_{1:t-1}$ are defined similarly. Then
\begin{eqnarray}
	\text{P}(\norm{S_t}\leq r, A_t=a|S_{1:t-1}, A_{1:t-1}, R_{1:t-1})\geq 1-\frac{C_0}{r}.
\end{eqnarray}
From triangle inequality, for $r>\norm{s}$, 
\begin{eqnarray}
	\text{P}(\norm{S_t-s}\leq r, A_t=a|S_{1:t-1}, A_{1:t-1}, R_{1:t-1})\geq 1-\frac{C_0}{r-\norm{s}}.
\end{eqnarray}
Hence
\begin{eqnarray}
	\text{P}(\rho_0(s, a)>r)&=&\text{P}(n(s, a, r)<k)\nonumber\\
	&=&\text{P}(T-n(s, a, r)\geq T-k)\nonumber\\
	&\leq & \text{P}(T-n(s, a, r)\geq \frac{1}{2} T)\nonumber\\
	&\leq & e^{-TC_0/(r-\norm{s})}\left(\frac{eT\frac{C_0}{r-\norm{s}}}{\frac{1}{2}T}\right)^{T/2}\nonumber\\
	&\leq & \left(\frac{2eC_0}{r-\norm{s}}\right)^{T/2}.
	\label{eq:largerbound}
\end{eqnarray}
Based on \eqref{eq:r0bound} and \eqref{eq:largerbound}, define $u=\max\{2\norm{s}, 8eC_0 \}$, then
\begin{eqnarray}
	\mathbb{E}\left[\max_a\rho_0(s, a)\right] &=&\int_0^\infty \text{P}\left(\max_a\rho_0(s, a)>r\right) dr\nonumber\\
	&\leq & \int_0^{r_0(s)} dr+\int_{r_0(s)}^u |\mathcal{A}|e^{-(1-\ln 2)k} dr +\int_u^\infty |\mathcal{A}|\left(\frac{2eC_0}{r-\norm{s}}\right)^{T/2} dr\nonumber\\
	&\leq & r_0(s) + u|\mathcal{A}|e^{-(1-\ln 2)k}+\frac{|\mathcal{A}|}{\frac{1}{2}T-1} (4eC_0)^{T/2}u^{1-T/2}\nonumber\\
	&\leq & r_0(s) + \max\left\{2\norm{s}, 8eC_0 \right\}|\mathcal{A}|e^{-(1-\ln 2)k}+\frac{2|\mathcal{A}|\max\{2\norm{s}, 8eC_0 \}}{T-2}2^{-T/2}\nonumber\\
	&\leq & r_0(s)+C_1(\norm{s}+1)|\mathcal{A}|e^{-(1-\ln 2)k},
\end{eqnarray}
for some constant $C_1$.

\textbf{Case 2: $r_0(s)>D$.} Now \eqref{eq:r0bound} does not hold. We only use the high probability bound for large $r$:
\begin{eqnarray}
	\mathbb{E}\left[\max_a\rho_0(s, a)\right]&=&\int_0^\infty \text{P}\left(\max_a \rho_0(s, a)> r\right) dr\nonumber\\
	&=&\int_0^u 1dr +\int_u^\infty |\mathcal{A}|\left(\frac{2eC_0}{r-\norm{s}}\right)^{T/2} dr\nonumber\\
	&=&u + \frac{|\mathcal{A}|}{\frac{1}{2}T-1} (4eC_0)^{T/2}u^{1-T/2}\nonumber\\
	&\leq & C_2(\norm{s} + 1)
	\label{eq:rholowpdf}
\end{eqnarray}
for some constant $C_2$. Note that the condition $g(s)\geq 3mk/(\pi_0\alpha v_dD^dT)$ in the statement of Lemma \ref{lem:rhotail} is exactly $r_0(s)\leq D$. Therefore, combining case 1 and 2, the proof of Lemma \ref{lem:rhotail} is complete.

\subsection{Proof of Lemma \ref{lem:rhotail2}}\label{sec:rhotail2}
The proof of Lemma \ref{lem:rhotail2} is based on Lemma \ref{lem:rhotail}.
\begin{eqnarray}
	\mathbb{E}\left[\max_a \rho_0(S', a')|s, a\right] &=& \int p(s'|s, a)\mathbb{E}\left[\max_{a'} \rho_0(s',a')|s,a\right] ds'\nonumber\\
	&\leq & \int_{r_0(s')\leq D}p(s'|s, a)g^{-\frac{1}{d}}(s')\left(\frac{3mk}{\pi_0\alpha v_d T}\right)^\frac{1}{d} ds'\nonumber\\
	&&+C_1(\mathbb{E}[\norm{S'}|s, a]+1)|\mathcal{A}| e^{-(1-\ln 2)k}\nonumber\\
	&&+\int_{r_0(s')>D} C_2(\norm{s'}+1)p(s'|s, a)ds'\nonumber\\
	&:=&I_1+I_2+I_3.
\end{eqnarray}
For $I_1$, from Assumption \ref{ass:tail}(e'),
$\int p(s'|s, a)g^{-1/d}(s')d\mathbf{s'}\leq C_g$, thus
\begin{eqnarray}
	I_1\leq C_g\left(\frac{3mk}{\pi_0\alpha v_d T}\right)^\frac{1}{d}.
\end{eqnarray}
For $I_2$, from Assumption \ref{ass:tail}(g'), $\mathbb{E}[\norm{S'}|s, a]\leq C_0$. Thus
\begin{eqnarray}
	I_2\leq C_1(C_0+1)|\mathcal{A}|e^{-(1-\ln 2)k}.
\end{eqnarray}
For $I_3$, $r_0(s')>D$ implies $g(s') < 3mk/(\pi_0\alpha v_d D^d T)$. From Assumption \ref{ass:tail}(e'),
\begin{eqnarray}
	I_3\leq C_2C_g \left(\frac{3mk}{\pi_0\alpha v_dD^dT}\right)^{\frac{1}{d}}.
\end{eqnarray}
Combine these three terms, 
\begin{eqnarray}
	\mathbb{E}\left[\max_{a'}\rho_0(\mathbf{S'}, a')|s, a\right]\leq C_3\left(\frac{k}{T}\right)^{\frac{1}{d}}
\end{eqnarray}
for some constant $C_3$. The proof is complete.

\section{Proof of Theorem \ref{thm:online}}\label{sec:onlinepf}
This section focuses on the online method. The proof begins with defining an event $E$.
\begin{defi}
	Let $E$ be the event such that the following conditions hold:
	
	1) For all $s\in \mathcal{S}$, $a\in \mathcal{A}$ and $t\leq T$, 
	\begin{eqnarray}
		\left|\frac{1}{k(t)}\sum_{j\in \mathcal{N}_t(s,a)}U_j\right|\leq \frac{\sigma_U}{\sqrt{k(t)}}\ln T,
		\label{eq:cond1}
	\end{eqnarray}
	with $U_j$ defined in \eqref{eq:uj} and $\sigma_U$ defined in \eqref{eq:sigmau};
	
	2) For all $s\in \mathcal{S}$, $a\in \mathcal{A}$ and $t_c\leq t\leq T$,
	\begin{eqnarray}
		\rho_t(s,a)\leq 2r_t,
		\label{eq:cond2}
	\end{eqnarray}
	with $r_t$ defined in \eqref{eq:rt};
	
	3) For all $1\leq t\leq T$, 
	\begin{eqnarray}
		|W_t|\leq \sigma \ln T.
		\label{eq:cond3}
	\end{eqnarray}	
\end{defi}

From \eqref{eq:noise2} and \eqref{eq:rho2} in Lemmas \ref{lem:noise} and \ref{lem:rho}, the probability of violating conditions 1) or 2) converges to zero with increase of $T$. Condition 3) can also be proved easily. From Assumption \ref{ass:main}(b), $$P(W_t>\sigma \ln T)\leq e^{-\ln^2 T/2},$$ hence
\begin{eqnarray}
	\text{P}\left(\underset{t}{\max}|W_t|>\sigma \ln T\right)\leq 2Te^{-\frac{1}{2}\ln^2 T}.
\end{eqnarray}

The above result indicates that $\text{P}(E^c) = o(1)$. Now it remains to bound the error under $E$.

Recall \eqref{eq:qonline} and \eqref{eq:Qonline},
\begin{eqnarray}
	q_t(s,a) - Q^*(s,a) &=& \frac{1}{k(t)}\sum_{j\in \mathcal{N}_t(s,a)} (Q(j)-Q^*(s,a))\nonumber\\
	&=&  \frac{1}{k(t)}\sum_{j\in \mathcal{N}_t(s,a)}(Q(j)-Q^*(S_j, a)) + \frac{1}{k(t)}\sum_{j\in \mathcal{N}_t(s,a)}(Q^*(S_j, a) - Q^*(s,a)),
\end{eqnarray}
and
\begin{eqnarray}
	Q(j) - Q^*(S_j, A_j)&=& R_j+\gamma \underset{a}{\max} q_j(S_{j+1}, a)- Q^*(S_j, A_j)\nonumber\\
	&=& R_j + \gamma \underset{a}{\max}q_j(S_{j+1}, a) - r(S_j, A_j)-\gamma \mathbb{E}\left[\underset{a}{\max} Q^*(S',a)|S_j, A_j\right]\nonumber\\
	&=& W_j +\gamma \left[\underset{a}{\max}q_j(S_{j+1}, a)-\mathbb{E}\left[\underset{a}{\max}Q^*(S',a)|S_j, A_j\right]\right]\nonumber\\
	&=& U_j + \gamma \left[\underset{a}{\max} q_j(S_{j+1},a)-\underset{a}{\max}Q^*(S_{j+1}, a)\right],
	\label{eq:Qdiff}
\end{eqnarray}
in which the last step uses \eqref{eq:uj}. Define
\begin{eqnarray}
	\Delta_t=\|q_t-Q^*\|_\infty.
\end{eqnarray}
Then for $t\geq t_c$, in which $t_c$ is defined in \eqref{eq:tc}, under the event $E$, we have
\begin{eqnarray}
	&&|q_t(s,a)-Q^*(s,a)|\nonumber\\
	&\leq& \left|\frac{1}{k(t)}\sum_{j\in \mathcal{N}_t(s,a)}(Q(j)-Q^*(S_j, a))\right| + \frac{1}{k(t)}\sum_{j\in \mathcal{N}_t(s,a)}\left|Q^*(S_j, a) - Q^*(s,a)\right|\nonumber\\
	&\overset{(a)}{\leq} & \left|\frac{1}{k(t)}\sum_{j\in \mathcal{N}_t(s,a)}U_j\right|+\gamma\left|\frac{1}{k(t)}\sum_{j\in \mathcal{N}_t(s,a)}\left[\underset{a'}{\max}q_j(S_{j+1}, a') - \underset{a'}{\max} Q^*(S_{j+1}, a')\right]\right| + L\rho_t(s,a)\nonumber\\
	&\overset{(b)}{\leq} & \frac{\sigma_U}{\sqrt{k(t)}}\ln T + \gamma\underset{\beta t\leq i<t}{\max}\Delta_i + 2Lr_t,
	\label{eq:qdiff}
\end{eqnarray}
in which (a) uses \eqref{eq:Qdiff} for the first two terms, and Lemma \ref{lem:lipschitz} and \eqref{eq:rhot} for the last term. (b) uses \eqref{eq:cond1} and \eqref{eq:cond2}. Recall that $r_t$ has been defined in \eqref{eq:rt}. Take supremum over \eqref{eq:qdiff}, for all $t\geq t_c$, under $E$,
\begin{eqnarray}
	\Delta_t\leq \sigma_U k(t)^{-\frac{1}{2}}\ln T+C_1\left(\frac{k(t)}{(1-\beta)t}\right)^\frac{1}{d}+\gamma\underset{\beta t\leq i<t}{\max}\Delta_i,
	\label{eq:transition}
\end{eqnarray}
in which
\begin{eqnarray}
	C_1=2L\left(\frac{3m}{\pi_0c\alpha v_d}\right)^\frac{1}{d}.
\end{eqnarray}
Let 
\begin{eqnarray}
	C(\gamma, \beta)=\max\left\{t_c^\frac{1}{d+2}\frac{R+\sigma}{1-\gamma}, \frac{(\sigma_U+C_1)(1-\beta)^{-\frac{1}{d+2}}}{1-\gamma \beta^{-\frac{1}{d+2}}}  \right\}.
	\label{eq:C}
\end{eqnarray}
We prove that if $E$ is true,
\begin{eqnarray}
	\Delta_t\leq C(\gamma, \beta)t^{-\frac{1}{d+2}}\ln T
	\label{eq:resonline}
\end{eqnarray}  
for $t=1,\ldots, T$, by induction.

\textbf{Case 1:  $t< t_c$}. From \eqref{eq:Qonline} and \eqref{eq:qonline},
\begin{eqnarray}
	\underset{1\leq t\leq T}{\max} Q(t)&\leq & \underset{1\leq t\leq T}{\max} R_t + \gamma \underset{1\leq t\leq T}{\max} Q(t)\nonumber\\
	&\leq & R+\sigma \ln T+ \gamma \underset{1\leq t\leq T}{\max} Q(t),
\end{eqnarray}
in which the last inequality comes from condition \eqref{eq:cond3}. Hence
\begin{eqnarray}
	\underset{1\leq t\leq T}{\max} Q(t)\leq \frac{R+\sigma \ln T}{1-\gamma},
\end{eqnarray}
and
\begin{eqnarray}
	\Delta_t=\|q_t-Q\|_\infty\leq \underset{s,a}{\sup} \max\{q_t(s,a), Q^*(s,a)\}\leq \frac{R+\sigma \ln T}{1-\gamma}\leq C(\gamma, \beta)t_c^{-\frac{1}{d+2}}\ln T,
\end{eqnarray}
in which the last step uses \eqref{eq:C}.

\textbf{Case 2: $t\geq t_c$.} Recall that $k(t)=\lceil ((1-\beta)t)^{2/(d+2)}\rceil$. We prove \eqref{eq:resonline} by induction. From now on, suppose that \eqref{eq:resonline} holds for steps $1,\ldots, t-1$, then for the $t$-th step, from \eqref{eq:transition},
\begin{eqnarray}
	\Delta_t &\leq& \sigma_U ((1-\beta) t)^{-\frac{1}{d+2}}\ln T+C_1\left(\frac{((1-\beta) t)^{2/(d+2)} + 1}{(1-\beta) t}\right)^\frac{1}{d}+\gamma C(\gamma, \beta)(\beta t)^{-\frac{1}{d+2}}\ln T\nonumber\\
	&\leq & (\sigma_U+C_1)((1-\beta) t)^{-\frac{1}{d+2}}+\gamma C(\gamma, \beta)(\beta t)^{-\frac{1}{d+2}}\ln T\nonumber\\
	&\leq & C(\gamma, \beta)t^{-\frac{1}{d+2}}\ln T,
\end{eqnarray}
in which the last step uses \eqref{eq:C}. 

Now we have proved that if $E$ is true, then \eqref{eq:resonline} holds. Moreover, $E$ is not true with a probability converging to zero with $T$ increases. Now it remains to pick $\beta$ that minimizes $C(\gamma, \beta)$. Let $\beta = \gamma^{(d+2)/(d+3)}$, then from \eqref{eq:C},
\begin{eqnarray}
	C(\gamma, \beta) \leq \max\left\{\frac{t_c^\frac{1}{d+2}(R+\sigma)}{1-\gamma}, (\sigma_U+C_1) (1-\gamma^\frac{d+2}{d+3})^{-\frac{d+3}{d+2}}\right\} .
\end{eqnarray}
It is straightforward to show the following inequality:
\begin{eqnarray}
	\gamma^\frac{d+2}{d+3}\leq 1-\frac{d+2}{d+3}(1-\gamma).
\end{eqnarray}
Therefore $C(\gamma, \beta)\lesssim (1-\gamma)^{-(d+3)/(d+2)}$. Recall that $\epsilon=\norm{q_T-Q}_\infty$. Therefore
\begin{eqnarray}
	\epsilon\lesssim (1-\gamma)^{-\frac{d+3}{d+2}}T^{-\frac{1}{d+2}}\ln T.
\end{eqnarray}
The corresponding sample complexity is
\begin{eqnarray}
	T=\tilde{O}\left(\frac{1}{\epsilon^{d+2} (1-\gamma)^{d+3}}\right).
\end{eqnarray}

The proof of Theorem \ref{thm:online} is complete.

\section{Proof of Theorem \ref{thm:onlinetail}}\label{sec:onlinetail}

From \eqref{eq:qdiff}, for any state action pair $(s, a)$,
\begin{eqnarray}
	&&q_t(s, a)-Q^*(s, a)-\left|\frac{1}{k(t)}\sum_{j\in \mathcal{N}_t(s, a)} U_j\right|-L\rho_t(s, a)\nonumber\\
	&&\leq \frac{\gamma}{k(t)}\left|\sum_{j\in \mathcal{N}_t(s, a)} \left[\max_{a'} q_j(S_{j+1}, a')-\max_{a'} Q^*(S_{j+1}, a')\right] \right|.
	\label{eq:deltaub}
\end{eqnarray}
Define
\begin{eqnarray}
	\Delta_t(s):=\max_a \left[|q_t(s, a)-Q^*(s, a)|-\left|\frac{1}{k(t)}\sum_{j\in \mathcal{N}_t(s, a)} U_j\right|-L\rho_t(s, a)\right].
	\label{eq:deltatsdef}
\end{eqnarray}
Then for any $s, a$,
\begin{eqnarray}
	|q_t(s, a)-Q^*(s, a)|\leq \Delta_t(s)+\left|\frac{1}{k(t)}\sum_{j\in \mathcal{N}_t(s, a)} U_j\right| + L\rho_t(s, a).
	\label{eq:qtub}
\end{eqnarray}
From \eqref{eq:deltaub},
\begin{eqnarray}
	\Delta_t(s) &=& \max_a \left[q_t(s, a)-Q^*(s, a)-\left|\frac{1}{k(t)}\sum_{j\in \mathcal{N}_t(s, a)} U_j\right|\right]\nonumber\\
	&\leq & \frac{\gamma}{k(t)}\max_a\left|\sum_{j\in \mathcal{N}_t(s, a)}\left[\max_{a'} q_j(S_{j+1}, a') - \max_{a'} Q^*(S_{j+1}, a')\right]\right|\nonumber\\
	&\leq & \frac{\gamma}{k(t)}\max_a \sum_{j\in \mathcal{N}_t(s, a)} \max_{a'} \left|q_j(S_{j+1}, a')-Q^*(S_{j+1}, a')\right|\nonumber\\
	&\leq & \frac{\gamma}{k(t)}\max_a\sum_{j\in \mathcal{N}_t(s, a)}\max_{a'} \left[\Delta_j(S_{j+1})+\frac{1}{k(j)}\left|\sum_{l\in \mathcal{N}_j(S_{j+1}, a')} U_l\right|+L\rho_j(S_{j+1}, a')\right],
\end{eqnarray}
in which the last step comes from \eqref{eq:qtub}. Define
\begin{eqnarray}
	\Delta_t:=\mathbb{E}[\max_s \Delta_t(s)].
	\label{eq:deltatdf}
\end{eqnarray}
Then
\begin{eqnarray}
	\Delta_t \leq \gamma \underset{\beta t\leq j < t}{\max}\Delta_j + \underset{\beta t\leq j < t}{\max}\frac{1}{k(j)} \mathbb{E}\left[\underset{s, a}{\sup}\left|\sum_{l\in \mathcal{N}_j(s, a)} U_l\right|\right]+\underset{\beta t\leq j < t}{\max}L \mathbb{E}\left[\underset{s, a}{\sup}\rho_j(S', a')|s, a\right],
	\label{eq:deltatstar}
\end{eqnarray}
in which $S'$ is a random state following distribution $p(\cdot|s, a)$.

It remains to bound the second and the third term. We show the following lemmas.

\begin{lem}\label{lem:supuonline}
	\begin{eqnarray}
		\mathbb{E}\left[\underset{s, a}{\sup}\left|\sum_{l\in \mathcal{N}_t(s, a)} U_l\right|\right] \leq \sqrt{\frac{2\sigma_U^2}{k(t)}\ln(d(1-\beta)^{2d} T^{2d}|\mathcal{A}|)}+\sqrt{\frac{2\pi \sigma_U^2}{k(t)}}.
	\end{eqnarray}
\end{lem}
The proof of Lemma \ref{lem:supuonline} is shown in Appendix \ref{sec:supuonline}.
\begin{lem}\label{lem:rhotailonline}
	With $t\geq 3m/(1-\beta)$, under Assumption \ref{ass:tail}, if 
	\begin{eqnarray}
		g(s) \geq \frac{3mk}{(1-\beta)\pi_0\alpha v_dD^dT},
	\end{eqnarray}
	then
	\begin{eqnarray}
		\mathbb{E}\left[\max_a\rho_t(s, a)\right] \leq \left(\frac{3mk}{(1-\beta)\pi_0\alpha v_d Tg(s)}\right)^\frac{1}{d} + C_1 (\norm{s}+1)|\mathcal{A}| e^{-(1-\ln 2)k}.
	\end{eqnarray}
	Otherwise
	\begin{eqnarray}
		\mathbb{E}\left[\max_a \rho_t(s, a)\right] \leq C_2(\norm{s}+1).
	\end{eqnarray}
\end{lem}
The proof of Lemma \ref{lem:rhotailonline} is shown in Appendix \ref{sec:rhotailonline}.
\begin{lem}
	For any state action pairs $s, a$, let $S'$ be a random state following distribution $p(\cdot|s, a)$, then
	\begin{eqnarray}
		\mathbb{E}\left[\underset{s, a'}{\sup}\rho_t(S', a')|s, a\right] \leq C_3\left(\frac{k}{(1-\beta) t}\right)^\frac{1}{d}.
	\end{eqnarray}
\end{lem}
The proof just follows that of Lemma \ref{lem:rhotail2}.

Based on these lemmas, from \eqref{eq:deltatstar},
\begin{eqnarray}
	\Delta_t&\leq& \gamma \underset{\beta t\leq j<t}{\max}\Delta_j + \underset{\beta t\leq j<t}{\max}\left[\sqrt{\frac{2\sigma_U^2}{k(j)}\ln(d(1-\beta)^{2d} T^{2d}|\mathcal{A}|)} + \sqrt{\frac{2\pi \sigma_U^2}{k(j)}}\right]+LC_3\underset{\beta t \leq j< t}{\max}\left(\frac{k(j)}{(1-\beta)j}\right)^\frac{1}{d}\nonumber\\
	&\leq & \gamma \underset{\beta t\leq j<t}{\max}\Delta_j+LC_3\underset{\beta t \leq j< t}{\max}\left(\frac{k(j)}{(1-\beta)j}\right)^\frac{1}{d}+C_4\underset{\beta t\leq j<t}{\max}\frac{1}{\sqrt{k(j)}}\ln T,
\end{eqnarray}
for some constant $C_4$.

Recall that for the case with bounded support, we have defined $C(\gamma, \beta)$ in \eqref{eq:C}. For the unbounded state space, now define
\begin{eqnarray}
	C'(\gamma, \beta) = \frac{(LC_3+C_4) (1-\beta)^{-\frac{1}{d+2}}\beta^{-\frac{1}{d+2}}}{1-\gamma \beta^{-\frac{1}{d+2}}}.
	\label{eq:Cp}
\end{eqnarray}
Then following arguments similar to Appendix \ref{sec:onlinepf}, it can be shown that
\begin{eqnarray}
	\Delta_t\leq C(\gamma, \beta) t^{-\frac{1}{d+2}}\ln T.
\end{eqnarray}
It remains to select $\beta$. Compared with the case with a bounded support, the most important difference is that now there is an additional $\beta^{-\frac{1}{d+2}}$ factor. The denominator in \eqref{eq:Cp} is required to be positive, thus $\gamma \beta^{-1/(d+2)}<1$, $\beta > \gamma^{d+2}$. Now we analyze the case that $1-\gamma$ is not large. To be more precise,  $\gamma \geq c_\gamma$ for some constant $c_\gamma\in (0,1)$, then $\beta \in (c_\gamma^{d+2}, 1)$, which is both upper and lower bounded by constants. To optimize \eqref{eq:Cp} asymptotically, it is enough to minimize $(1-\beta)^{-1/(d+2)}/ (1-\gamma \beta^{-1/(d+2)})$. The minimizer is $\beta = \gamma^{(d+2)/(d+3)}$. Then
\begin{eqnarray}
	C'(\gamma, \beta)\lesssim \left(\frac{1}{1-\gamma}\right)^\frac{d+3}{d+2},
\end{eqnarray}
and
\begin{eqnarray}
	\Delta_t\lesssim \left(\frac{1}{1-\gamma}\right)^\frac{d+3}{d+2}t^{-\frac{1}{d+2}}\ln T.
\end{eqnarray}
From \eqref{eq:deltatdf} and \eqref{eq:deltatsdef}, it can be shown that
\begin{eqnarray}
	\int \mathbb{E}\left[\max_a |q_T(s, a)-Q^*(s, a)|\right] f_\pi(s)ds\leq \left(\frac{1}{1-\gamma}\right)^\frac{d+3}{d+2} T^{-\frac{1}{d+2}}\ln T.
\end{eqnarray}
Recall that now $\epsilon=\int \mathbb{E}\left[\max_a |q_T(s, a)-Q^*(s, a)|\right] f_\pi(s)ds$, the sample complexity is
\begin{eqnarray}
	T=\tilde{O}\left(\frac{1}{(1-\gamma)^3 \epsilon^{d+2}}\right).
\end{eqnarray}

\subsection{Proof of Lemma \ref{lem:supuonline}}\label{sec:supuonline}
From \eqref{eq:onlineub},
\begin{eqnarray}
	\text{P}\left(\underset{s, a}{\sup}\left|\frac{1}{k(t)}\sum_{j\in \mathcal{N}_t(s, a)} U_j\right| > u\right) \leq d(1-\beta)^{2d} T^{2d} |\mathcal{A}| e^{-\frac{k(t) u^2}{2\sigma_U^2}}.
\end{eqnarray}
The remainder of the proof follows that of Lemma \ref{lem:largeu}. We omit the detailed steps for simplicity. Finally, we get
\begin{eqnarray}
	\mathbb{E}\left[\underset{s, a}{\sup}\left|\sum_{l\in \mathcal{N}_t(s, a)} U_l\right|\right] \leq \sqrt{\frac{2\sigma_U^2}{k(t)}\ln(d(1-\beta)^{2d} T^{2d}|\mathcal{A}|)}+\sqrt{\frac{2\pi \sigma_U^2}{k(t)}}.
\end{eqnarray}

\subsection{Proof of Lemma \ref{lem:rhotailonline}}\label{sec:rhotailonline}
The proof is similar to that of Lemma \ref{lem:rhotail}. Define
\begin{eqnarray}
	r_t(s) = \left(\frac{3mk}{(1-\beta)\pi_0 \alpha v_d t g(s)}\right)^\frac{1}{d}.
\end{eqnarray}
\textbf{Case 1: $r_t(s)\leq D$.} Recall the definition of $n_t(s, a, r)$ in \eqref{eq:ntdf}. Then \eqref{eq:pmass} becomes
\begin{eqnarray}
	\text{P}\left(\norm{S_{t+m}-s}\leq r_t(s), A_{t+m}=a|S_t, A_t\right)\geq \frac{3km}{(1-\beta) t}.
\end{eqnarray}
From Lemma \ref{lem:chernoff},
\begin{eqnarray}
	\text{P}(\rho_t(s, a)>2r_t(s))\leq e^{-(1-\ln 2)k}.
\end{eqnarray}
The remainder of the proof follows that of Lemma \ref{lem:rhotail}. We omit the detailed steps for simplicity. The final bound is
\begin{eqnarray}
	\mathbb{E}\left[\max_a \rho_t(s, a)\right] \leq r_t(s) + C_1(\norm{s} + 1)|\mathcal{A}| e^{-(1-\ln 2) k}.
\end{eqnarray}
\textbf{Case 2: $r_t(s) > D$.} Similar to \eqref{eq:rholowpdf},
\begin{eqnarray}
	\mathbb{E}\left[\max_a \rho_t(s, a)\right] \leq C_2(\norm{s} + 1).
\end{eqnarray}

\end{document}